\documentclass{article}

\usepackage{PRIMEarxiv}

\usepackage{xcolor}
\usepackage[normalem]{ulem} 

\usepackage[ruled,vlined]{algorithm2e}

\usepackage[utf8]{inputenc} 
\usepackage[T1]{fontenc}    
\usepackage{hyperref}       
\usepackage{xr-hyper}
\usepackage{url}            
\usepackage{booktabs}       
\usepackage{amsfonts}       
\usepackage{nicefrac}       
\usepackage{microtype}      
\usepackage{xcolor}         
\usepackage{blindtext}
\usepackage{amsmath}
\usepackage[nameinlink,capitalize,noabbrev]{cleveref}
\usepackage[ruled,vlined]{algorithm2e}
\usepackage{amsthm}
\usepackage{graphicx}
\usepackage{subcaption}
\usepackage{mathtools}
\usepackage{dsfont}
\usepackage{thmtools,thm-restate}
\usepackage{natbib}
\usepackage{tikz}
\usepackage{pgf}
\usepackage{wrapfig}
\usepackage{amssymb}
\usepackage{verbatim}
\usepackage{etoc}
\usepackage{fontawesome}
\usepackage{threeparttable}

\usepackage{xcolor}

\newtheorem{theorem}{Theorem}
\newtheorem{lemma}[theorem]{Lemma}
\newtheorem{proposition}[theorem]{Proposition}

\theoremstyle{definition}

\newcommand{\iid}{\mathrel{\overset{\mathrm{iid}}{\sim}}}

\newcommand{\E}{\mathbb{E}}

\newcommand{\cO}{O}

\newcommand{\cF}{\mathcal{F}}
\newcommand{\R}{\mathbb{R}}

\usepackage{tikz}
\usetikzlibrary{external}
\usepackage[utf8]{inputenc} 
\usepackage[T1]{fontenc}    
\usepackage{hyperref}       
\usepackage{url}            
\usepackage{booktabs}       
\usepackage{amsfonts}       
\usepackage{nicefrac}       
\usepackage{microtype}      
\usepackage{lipsum}
\usepackage{fancyhdr}       
\usepackage{graphicx}       
\graphicspath{{media/}}     

\title{Efficient Logistic Regression with \\ Mixture of Sigmoids}

\authorA{%
\textbf{Federico Di Gennaro} \\
ETH Zürich \\
\texttt{fdigennaro@ethz.ch}
}

\authorB{%
\textbf{Saptarshi Chakraborty} \\
University of Michigan \\
\texttt{saptarsc@umich.edu}
}

\authorC{%
\textbf{Nikita Zhivotovskiy} \\
UC Berkeley \\
\texttt{zhivotovskiy@berkeley.edu}
}

\begin{document}
\maketitle

\begin{abstract}
This paper studies the Exponential Weights (EW) algorithm with an isotropic Gaussian prior for online logistic regression. We show that the near-optimal worst-case regret bound \(O(d\log(Bn))\) for EW, established by \citet{kakade_ng_bayesianalg} against the best linear predictor of norm at most \(B\), can be achieved with total worst-case computational complexity \(\tilde O(B^3 n^5)\). This substantially improves on the \(O(B^{18}n^{37})\) complexity of prior work achieving the same guarantee \citep{foster2018logistic}. Beyond efficiency, we analyze the large-\(B\) regime under linear separability: after rescaling by \(B\), the EW posterior converges as \(B\to\infty\) to a standard Gaussian truncated to the version cone. Accordingly, the predictor converges to a \emph{solid-angle vote} over separating directions and, on every fixed-margin slice of this cone, the mode of the corresponding truncated Gaussian is aligned with the hard-margin SVM direction. Using this geometry, we derive non-asymptotic regret bounds showing that once \(B\) exceeds a margin-dependent threshold, the regret becomes independent of \(B\) and grows only logarithmically with the inverse margin. Overall, our results show that EW can be both computationally tractable and geometrically adaptive in online classification.
\end{abstract}


\section{Introduction}
\label{sec:introduction}

Logistic regression is arguably the most well-known method for probability estimation and binary classification. In this paper, we study the online version of the problem, where a learner interacts with an environment over $n$ rounds and updates its prediction (usually in the form of a logit score $\hat{z}_t$ induced by a linear predictor) \emph{on the fly} given the data $x_t$ at hand\footnote{Usually assumed to be normalized.} \citep{cesa2006prediction}. The true label $y_t\in\{\pm1\}$ is then revealed and the learner incurs the logistic loss $\ell_{\log}(z,y)=\log(1+\exp(-yz))$. Performance is measured by \emph{regret}, i.e., the excess cumulative loss relative to the best bounded-norm linear predictor chosen in hindsight:
\begin{equation}
\label{eq: regret definition}
R_n = \sum_{t=1}^n \ell_{\log}(\hat z_t,y_t)
      - \inf_{\theta\in\cF}\sum_{t=1}^n \ell_{\log}(\langle \theta,x_t\rangle,y_t),
\end{equation}
where $\cF=\{\theta\in\R^d:\|\theta\|\le B\}$. In general, the goal in adversarial online learning is to design sequential predictors with sublinear regret, namely $R_n=o(n)$. In online logistic regression, the mixability of the logistic loss allows one to achieve a regret rate of order $O(\log n)$, and several works have studied how this optimal dependence on $n$ can be obtained \citep{kakade_ng_bayesianalg,vovk2001competitive,JMLRvanerven15a}.

A deeper question is how the regret scales with the comparator radius $B$. A key difficulty is that, when $B=\Omega(\log n)$, any \emph{proper} algorithm --- an algorithm whose predictions correspond to a $\theta\in\cF$ --- can suffer $\Omega(\sqrt{n})$ regret in the worst case \citep{hazan2014logistic}. To overcome this lower bound, \citet{kakade_ng_bayesianalg, foster2018logistic} show the importance of using \emph{improper} strategies to achieve the near-minimax worst-case rate $O(d\log(Bn))$. However, the procedures achieving this optimal rate are computationally prohibitive, and existing efficient methods typically lose the logarithmic dependence on $B$ (see \cref{tab:algorithms}). This raises a first question:

\begin{center}
\begin{minipage}{0.9\columnwidth}
\centering
\textit{\textbf{Q1.} Can we achieve the optimal $O(d\log(Bn))$ regret with a more efficient predictor?}
\end{minipage}
\end{center}

Second, in some relevant scenarios, bounding the comparator norm is unnatural or even vacuous. For instance, when the data are linearly separable, the logistic loss has no finite minimizer: its infimum is zero and is approached only as $\|\theta\|\to\infty$ along separating directions \citep{albert1984existence,Soudry18}. In such a regime, it is therefore natural to ask whether one can allow an unbounded comparator norm. However, it is unclear how existing estimators behave in this regime, and whether they still yield non-vacuous guarantees (cf. \cref{tab:algorithms}). This leads to our second question:

\begin{center}
\begin{minipage}{0.9\columnwidth}
\centering
\textit{\textbf{Q2.} Can a single predictor be used to get non-vacuous regret bounds for both the case of bounded and unbounded comparator norm $B$?}
\end{minipage}
\end{center}

We address the above questions through a \emph{mixture-of-sigmoids} predictor arising from the \emph{Exponential Weights} (EW) algorithm with an isotropic Gaussian prior. Because the Gaussian prior makes the posterior potential globally strongly convex while keeping the parameter space unconstrained, we can draw on recent results for unconstrained sampling from strongly log-concave distributions \citep{ChewiBook24}. This yields an efficient and computable implementation of EW that still provably attains the $O(d\log(Bn))$ regret rate.

Moreover, we analyze how this \emph{mixture of sigmoids} behaves beyond the standard bounded-comparator regime. We show that EW naturally adapts (i.e. there is no need to modify the prediction rule) to the case of linearly separable data, when we naturally require unbounded comparator norm. In this scenario, as $B\to\infty$, the EW predictor converges to a \emph{solid-angle vote} over separating directions. Interestingly, such a limiting estimator has a clear connection with the hard-margin Support Vector Machine (SVM) solution and, intuitively, this motivates the fact that a margin-dependent (and thus non-vacuous) regret bound when data are linearly separable and $B=\infty$ is achievable.

\subsection{Summary of key contributions}
The contributions of this paper are twofold, and answer the questions raised above:
\begin{itemize}
    \item We show that Gaussian-prior EW achieves the near-optimal regret rate $O(d\log(Bn))$, and we provide an explicit implementation whose regret is provably of the same order, with total worst-case complexity $\tilde O(B^3 n^5)$. This substantially improves over the $O(B^{18}n^{37})$ computational complexity of \citep{foster2018logistic} for the same regret rate. The key enabling feature is that the Gaussian-prior EW posterior is unconstrained on $\mathbb{R}^d$, which allows us to leverage unconstrained sampling machinery rather than projection-based methods such as the \emph{Projected Langevin Algorithm} \citep{bubeck2018sampling} used in \citep[Appendix B]{foster2018logistic}.

    \item In the separable regime, we characterize the limit as $B\to\infty$: the EW predictive distribution converges to a solid-angle vote induced by a standard Gaussian truncated to the version cone. Moreover, on every fixed-margin slice of this cone, the mode of the limiting truncated Gaussian is aligned with the hard-margin SVM direction. We also prove non-asymptotic regret bounds showing that, once $B$ exceeds a margin-dependent threshold, the regret becomes independent of $B$ and depends only logarithmically on the inverse margin.
\end{itemize}

\subsection{Related Work}
\label{sec:related_work}

Logistic regression is a fundamental supervised learning framework for estimating conditional probabilities. In the classical batch setting, given data $\{(x_i,y_i)\}_{i=1}^n$ with labels $y_i\in\{\pm1\}$, the probability is modeled as $\mathbb{P}(Y=1\mid x;\theta)=\sigma(\langle\theta,x\rangle)$, where $\sigma(s)=(1+e^{-s})^{-1}$ is the sigmoid function. The unknown parameter vector $\theta\in\mathbb{R}^d$ is typically estimated via the \emph{Maximum Likelihood Estimation} (MLE) \citep{McCullaghNelder1989,HastieTibshiraniFriedman2009,Murphy2012}; the MLE selects $\hat\theta_{\text{MLE}}$ to maximize the likelihood $\prod_{i=1}^n \sigma(y_i\langle\theta,x_i\rangle)$ or, equivalently, to minimize the cumulative logistic loss. A natural extension of MLE to the sequential setup, is to compute the MLE on the history available at time $t$, a strategy known as \emph{Follow-The-Leader} (FTL) or, with $\ell_2$-regularization, \emph{Follow-The-Regularized-Leader} (FTRL) \citep{ShalevShwartz2012,Hazan2016}. Other standard methods for online logistic regression are first- or second-order algorithms from \emph{Online Convex Optimization} such as Online Gradient Descent \citep{zinkevich2003online} and Online Newton Step \citep{hazan2007logarithmic}. All these approaches are \emph{proper}, meaning they restrict predictions to correspond to a single parameter $\theta$ inside the comparator class $\cF$. However, proper prediction can be fundamentally suboptimal for online logistic regression, where a logarithmic regret with respect to the number of rounds $n$ is achievable. Indeed, \citet{hazan2014logistic} proved that when $B=\Omega(\log n)$, any proper algorithm can incur $\Omega(\sqrt{n})$ regret and this motivates \emph{improper} strategies that predict with mixtures rather than a single comparator \citep{foster2018logistic}.

\begin{table*}[t]
  \centering
  \small
  \setlength{\tabcolsep}{8pt} 
  \begin{tabular}{lcc}
    \toprule
    \textbf{Algorithm} & \textbf{Regret} & \textbf{Total complexity} \\
    \midrule
    OGD \citep{zinkevich2003online} & $B\sqrt{n}$ & $n$ \\
    ONS \citep{hazan2007logarithmic} & $de^{B}\log n$ & $n$ \\
    EW with uniform prior on bounded support \citep{foster2018logistic} & ${d\log(Bn)}$ & $B^{18}n^{37}$ \\
    AIOLI \citep{rudi_binary_efficient_logreg} & $dB\log(n)$ & $nB^2$ \\
    EW with Gaussian prior (Ours) & $d \log(Bn)$ & $B^3n^5$ \\
    \bottomrule
  \end{tabular}
  \caption{Regret bounds (in $O(\cdot)$ notation) and total complexity (in $\tilde{\cO}(\cdot)$ notation) on online (binary) logistic regression.}
  \label{tab:algorithms}
\end{table*}

\paragraph{Efficient algorithms.}
\citet{foster2018logistic} proposed an improper algorithm based on Vovk's aggregating strategy \citep{vovk1990aggregating} with uniform prior (constrained on the $\ell_2$ ball of radius $B$) that achieves the near-minimax regret rate $O(d\log(Bn))$. Related Bayesian approaches for online logistic regression were developed earlier by \citet{kakade_ng_bayesianalg}; more recently, \citet{shamir2020logistic} refined upper and lower bounds under different norm constraints ($\ell_2,\ell_1,\ell_\infty$) and scaling regimes, and \citet{wu2022precise} derived minimax bounds using constructive truncated Bayesian predictors with sharpened constants. Despite their statistical optimality, regret-optimal improper predictors are often difficult to implement efficiently. A key example is the algorithm by \citet{foster2018logistic}, whose prediction requires approximating a log-concave posterior supported on a bounded set; implementing this with projection-based samplers such as the \emph{Projected Langevin Algorithm} \citep{bubeck2018sampling} leads to a large worst-case complexity of order $\cO(B^{18}n^{37})$. This computational barrier has motivated a parallel line of work seeking algorithms with practical complexity, always at the expense of the optimal logarithmic dependence on $B$. The pioneering work in this area is by \citet{rudi_binary_efficient_logreg} that proposed \emph{AIOLI}, an algorithm with complexity of order $\cO\big(n(d^2+B^2\log(n))\big)$ achieving $\cO(dB\log(n))$ regret on online (binary) logistic regression. The idea underlying AIOLI is to use FTRL on a surrogate quadratic loss function with an additional improper regularization. While \citet{rudi_binary_efficient_logreg} established a computationally inexpensive polynomial time algorithm with $\log(n)$ regret, the logarithmic dependence of the regret with respect to the parameter $B$ is lost, and the regret is thus exponentially worse compared to the one in \citep{foster2018logistic}. While subsequent works focused on extending this result to $K$-class logistic regression \citep{agarwal2022efficient,jezequel2021mixability}, surprisingly the algorithm in \citep{foster2018logistic} remains the only one with logarithmic dependence with respect to the parameter $B$ (cf. \cref{tab:algorithms}).

\paragraph{Large-$B$ regime \& linear separability.}
The standard comparator framework imposes the constraint $\|\theta\|\le B$. However, this can become unnatural in the linearly separable regime: the infimum of the logistic loss is zero and is approached only as $\|\theta\|\to\infty$, while the unregularized maximum likelihood estimator may fail to exist \citep{albert1984existence}. Moreover, when the comparator norm is allowed to be arbitrarily large, the bounds of the algorithms in \cref{tab:algorithms} become vacuous. To the best of our knowledge, existing works do not analyze how these estimators behave in this regime. From the optimization perspective, \citet{Soudry18} showed that, on separable data, gradient descent on the logistic loss exhibits an implicit bias: the parameter norm diverges while its direction converges to the hard-margin SVM separator. While these results clarify the behavior of classical optimization dynamics, they do not directly provide an online-regret perspective on Bayesian aggregation methods in the $B\to\infty$ regime.

\section{Preliminaries}
\label{sec:preliminaries}

We consider online binary logistic regression in $\mathbb{R}^d$. At each round $t=1,\dots,n$, the learner observes $x_t \in \mathcal{X} \subseteq \mathbb{R}^d$ and outputs either a logit score $\hat z_t \in \mathbb{R}$ or a probability distribution $\hat{p}_t(x_t,y)$ over $y\in\{\pm 1\}$. Then, the true $y_t \in \{\pm 1\}$ is revealed and the learner incurs the logistic loss
\[
\ell_\mathrm{log}(\hat z_t,y_t) = \log\bigl(1+\exp(-y_t \hat z_t)\bigr) = -\log \sigma(y_t \hat z_t),
\]
where $\sigma(\cdot)$ is again the sigmoid function. Furthermore, a standard boundedness assumption $\|x_t\|\le R$ is made\footnote{cf. \citet{rudi_binary_efficient_logreg}.} for all $t$. We first start with the vanilla case of comparing the learner against the set of linear predictors $f_\theta(x)=\langle \theta, x\rangle$ with bounded norm $\|\theta\|\le B$. If the learner directly outputs a probability distribution $\hat{p}_t(x_t,y)$, we can rewrite the regret as
\begin{equation*}
\label{eq: regret definition EWA}
    R_n = -\sum_{t=1}^n \log\hat{p}_t(x_t,y_t)-\underset{\substack{\theta \in \mathbb{R}^d\\ \|\theta\|\leq B}}{\inf}\sum_{t=1}^n -\log\sigma(y_t\langle \theta,x_t\rangle).
\end{equation*}

\paragraph{Notation.} For a positive integer $k$, $[k]$ denotes the set $\{1,...,k\}$. We use $\log(\cdot)$ to denote the natural logarithm. For two distributions $\rho,\nu$ defined on the space $(\mathbb{R}^d, \mathcal{B}(\mathbb{R}^d))$, where $\mathcal{B}(\mathbb{R}^d)$ denotes the Borel-sigma algebra on $\mathbb{R}^d$, we use $d_{TV}(\rho,\nu)$ to denote their total variation distance. While not specified otherwise, we refer to $\|\cdot\|$ as the $\ell_2$ (Euclidean) norm in $\mathbb{R}^d$. We denote the Euclidean ball of radius $R>0$ as $B_R(\mathbb{R}^d)=\{x\in\mathbb{R}^d:\|x\|\le R\}\subset \mathbb{R}^d$. We adopt the usual asymptotic notations of $O(\cdot), \ \Omega(\cdot), \ \Theta(\cdot)$, where their corresponding \emph{tilde} versions just suppress extra poly-logarithmic factors. \footnote{In this work, both $d$ and $R$ are considered constant, so most of the times won't be displayed in the big-$O$ notation.}

\subsection{Exponential Weights Algorithm}
\label{subsec: EWA}
Exponential Weights (EW, cf. \cref{alg:EWA LogReg}) \citep{vovk1990aggregating} is a fundamental algorithm for sequential prediction that maintains a distribution over a set of \emph{experts} and updates it multiplicatively in response to the observed losses. In the context of online logistic regression, EW with prior $\pi_0$ on the parameter space and learning rate $\eta=1$ predicts at round $t$ according to the probability distribution on $y\in\{\pm1\}$ given by
\begin{equation}
\label{eq: estimators EWA}
\hat p_t(x_t,y)
=
\mathbb{E}_{\theta\sim\rho_t}\!\left[\sigma\bigl(y\langle\theta,x_t\rangle\bigr)\right],
\end{equation}
that is a \emph{mixture of sigmoids}. Here, $\rho_t(\theta) \propto \pi_0(\theta)\prod_{i=1}^{t-1} \sigma(y_i\langle x_i,\theta \rangle) $ is the \emph{posterior} obtained by reweighting the prior with the past log-loss. For infinite expert classes, $\pi_0$ is simply a continuous (probabilistic) prior on the parameter space. EW is also \emph{improper}: its prediction does not, in general, coincide with that of any bounded-norm linear predictor in the comparator class $\mathcal{F}$. Conceptually, the EW update can also be viewed as a Bayesian update over a mixture of experts \citep{kakade_ng_bayesianalg}. A common obstacle with such Bayesian predictors is their computational cost, since one must approximate the posterior distribution $\rho_t$ at each round of the sequential protocol. Accordingly, prior work on efficient online logistic regression has largely moved away from Bayesian algorithms because of this computational burden. Our key observation is that, with a Gaussian prior, the EW predictor is an \emph{unconstrained} mixture on $\mathbb{R}^d$. This makes the posterior amenable to unconstrained sampling methods, yielding a computationally tractable implementation while preserving the logarithmic dependence of the regret on both $n$ and $B$.

\section{Regret bounds and computational complexity}
\label{sec:ewa_regret_and_cheap_implementation}

In this section, we study the regret guarantee of EW on logistic loss and evaluate the computational complexity of its implementation. A standard regret analysis for EW leads to the regret bound in \cref{th: regret guarantee of traditional EWA}; the result is essentially present in \citep{kakade_ng_bayesianalg}, but we formulate and prove it for the sake of completeness.

\begin{restatable}{theorem}{EWArb}[Reformulated \citep{kakade_ng_bayesianalg}]
\label{th: regret guarantee of traditional EWA}
    The estimator in \cref{eq: estimators EWA} with an isotropic Gaussian prior $\pi_0(\theta)\sim\mathcal{N}(0,B^2I_d)$ satisfies a regret guarantee of order $\cO(d\log(Bn))$.
\end{restatable}

\begin{proof}[Sketch of proof]
    Using the fact that the log-loss is $1$-mixable, the result follows by combining two standard identities in the online learning literature. The first one is a result by \citet[Lemma 1]{vovk2001competitive} and the second one is a corollary of the Donsker-Varadhan variational formula \citep{donsker1975asymptotic}.
\end{proof}

Computing the EW predictors in \eqref{eq: estimators EWA} is nontrivial. As with other Bayesian algorithms, one must evaluate an expectation with respect to a posterior distribution $\rho_t$ that is known only up to its normalizing constant. Writing the posterior as $\rho_t(d\theta)\propto \exp\{-V_t(\theta)\}\,d\theta$, where $V_t$ denotes the posterior potential, the key point is that the Gaussian prior makes $V_t$ globally strongly convex while keeping the parameter space unconstrained on $\mathbb{R}^d$. This substantially simplifies posterior approximation and allows us to leverage recent non-asymptotic guarantees for unconstrained sampling from strongly log-concave distributions \citep{ChewiBook24}. Once an approximate posterior is available, we estimate the prediction by Monte Carlo averaging. Finally, to keep the resulting probabilities bounded away from zero, we apply a simple smoothing step. We now describe each of these ingredients in turn.

\textbf{1) Sampling.} As mentioned before, we cannot directly access samples from $\rho_t$ since its normalizing constant is unknown and difficult to compute. Thus, we show that one can instead sample from an approximate posterior $\tilde{\rho}_t$ without changing the regret rate, provided that $d_{TV}(\rho_t,\tilde{\rho}_t)\leq E_t$, where $E_t$ is the cumulative approximation error at round $t$, obtained from the fresh per-round budgets $\{\varepsilon_i\}_{i=1}^t$ and chosen so that using EW estimators based on $\tilde \rho_t$ only incurs an additional additive term in the regret. Now, notice that the posterior distribution
\[
\rho_t(d\theta)\ \propto\ \exp\Big\{\underbrace{\sum_{i=1}^{t-1} \log\sigma\big(y_i\langle x_i,\theta\rangle\big)
-\ \frac{1}{2B^2}\|\theta\|_2^2}_{-V_t(\theta)}\Big\}\,d\theta
\]
is strongly \emph{log-concave} for every $t\in[n]$. Indeed, thanks to the Gaussian prior, the potential $V_t(\theta)$ is $m$-strongly convex with $m=1/B^2$ and $L$-smooth with $L=\tfrac14 R^2(t-1)+1/B^2$ (see \cref{app-subsec: computational complexity}), so its condition number is $\kappa_t \coloneqq \tfrac{L}{m} = 1+\tfrac14 B^2R^2(t-1)=\cO\big(B^2R^2(t-1)\big)$. We therefore approximate $\rho_t$ by $\tilde{\rho}_t$ using the \emph{Metropolis-Adjusted Langevin Algorithm} (MALA) \citep{roberts2002langevin}, which has been shown \citep{metropolis_fast,chewi2021optimal,altschuler2024faster} to mix in a number of steps that depends only poly-logarithmically on the target TV accuracy when initialized from a warm start. To obtain such a warm start, we do not sample $\rho_t$ from scratch. Instead, we initialize from the previous-round approximation and bridge $\rho_{t-1}$ to $\rho_t$ through a sequence of intermediate tempered targets (similarly to \citep{neal2001annealed,del2006sequential}). The fresh round-$t$ accuracy budget $\varepsilon_t$ is split across the intermediate rungs as per-rung tolerances $(\varepsilon_j)_{j=1}^K$ with $\sum_{j=1}^K \varepsilon_j\le \varepsilon_t$. This yields $d_{\mathrm{TV}}(\rho_t,\tilde\rho_t)\le d_{\mathrm{TV}}(\rho_{t-1},\tilde\rho_{t-1})+\varepsilon_t$, and therefore $d_{\mathrm{TV}}(\rho_t,\tilde\rho_t)\le \sum_{i=1}^t \varepsilon_i$. 

\textbf{2) Monte Carlo Expectation.} Once we have access to the approximate distribution $\tilde{\rho}_t$, we also need to approximate the expectation $\mathbb{E}_{\theta\sim \tilde{\rho}_t}\bigl[\sigma\bigl(y\,\langle x_t,\theta\rangle\bigr)\bigr]$. To do so, we simply use a Monte Carlo average $\tfrac{1}{s_t}\sum_{i=1}^{s_t} \sigma(y\langle x_t, \theta_i\rangle)$ over $s_t$ samples $\theta_i \iid \tilde{\rho}_t$. To get those, we run $s_t$ independent copies of this bridged warm-start procedure, each initialized from its own previous-round state and using independent randomness; consequently, the resulting samples are i.i.d. with common law $\tilde\rho_t$. The number of samples $s_t$ required at time $t$ is again chosen so that the regret of EW is only increased by an additive constant.
    
\textbf{3) Smoothing.} Finally, we need to make sure that the Monte Carlo approximation of $\mathbb{E}_{\theta\sim \tilde{\rho}_t}\bigl[\sigma\bigl(y\,\langle x_t,\theta\rangle\bigr)\bigr]$ is bounded away from zero to control the log-loss perturbation via Chernoff's bound. To do so, we predict after \emph{smoothing} the Bernoulli parameter via the function $\mathrm{smooth}_{\alpha_t}(p) \colon \ p \mapsto (1-\alpha_t)p+\frac{\alpha_t}{2}, \ \alpha_t\in\left[0,\frac{1}{2}\right]$, with $p$ being a distribution on the binary outcome space $\{-1,1\}$ and $\alpha_t$ chosen such that the smoothing procedure does not affect the order of the regret. 

The resulting predictive estimator of the above described three-steps modification of the estimate in \cref{eq: estimators EWA} is
\begin{equation}
\label{eq: computable EWA estimators}
\tilde{p}^{\alpha_t,s_t}_t(x_t,y)
\;=\;
\mathrm{smooth}_{\alpha_t}\!\left(\frac{1}{s_t}\sum_{i=1}^{s_t}\sigma\bigl(y\langle x_t,\theta_{t,i}\rangle\bigr)\right),
\end{equation}
where $y\in\{\pm1\}$ and $\theta_{t,i} \stackrel{\text{iid}}{\sim} \tilde{\rho}_t$ with $d_{\mathrm{TV}}(\rho_t,\tilde{\rho}_t)\le E_t$ for $i=1,...,s_t$. If we denote $\varepsilon_t$ as the \emph{fresh} accuracy budget for the sampler at time step $t$, we can prove that with a suitable choice of parameters $(\alpha_t,\varepsilon_t,s_t)$ the regret remains of order $O\left(d\log(Bn)\right)$ with high probability. In addition, the total computational cost only depends on how fast we can get the approximated posterior $\tilde{\rho}_t$ and how many samples $s_t$ we use to approximate the expectation. We formalize these facts in the following result.

\begin{restatable}{theorem}{CEWAHighProbrb}
\label{th: regret guarantee of COMPUTABLE EWA}
Let $\delta>0$, let $(\alpha_t,s_t, \varepsilon_t, \delta_t)_{t=1}^n$ be the parameters of the estimator in \ref{eq: computable EWA estimators}, which satisfy
\[
\sum_{t=1}^n \delta_t \le \delta, \quad \alpha_t s_t \geq 16\log(1/\delta_t) \ \forall t=1,...,n.
\]
Then, with probability at least $1-\delta$ the regret bound of that estimator is
\begin{equation}
\begin{split}
\label{eq: regret computable EWA}
R_n \le O(&d\log(Bn)) + \sum_{t=1}^n 2\alpha_t + \sum_{t=1}^n \frac{2 \sum_{i=1}^t \varepsilon_i}{\alpha_t}
+ 4\sum_{t=1}^n \sqrt{\frac{\log(\nicefrac{1}{\delta_t})}{s_t\alpha_t}},
\end{split}
\end{equation}
with $\varepsilon_i$ being the new fresh accuracy budget at round $i$ and $\{\theta_{t,j}\}_{j=1}^{s_t}$ obtained from $s_t$ independent copies of the bridged warm-start MALA construction, each initialized from its own previous-round sample. Then, the per-round computational cost of computing $\tilde p_t^{\alpha_t,s_t}$ is of order $\tilde O(s_t \cdot Q_t)$ and MALA sampling strategy is such that $Q_t=\tilde{O}\left(\sqrt d\,\kappa_t\,RB\,\log\frac{1}{\varepsilon_t}\right)$.
\end{restatable}

\begin{proof}[Sketch of proof]
Start from the standard EW regret bound, which is of order $\cO(d\log(Bn))$. Then show that replacing the exact EW predictor by $\tilde p_t^{\alpha_t,s_t}$ only adds three terms: one from smoothing, one from using an approximate posterior, and one from Monte Carlo approximation (cf. \cref{lemma: regret additive constant TV-shift,lemma: regret additive constant smoothing,lemma: regret additive constant multiplicative chernoff}). For what concerns computational complexity, at round $t$, we approximate $\rho_t$ by warm-starting from the previous round and bridging $\rho_{t-1}$ to $\rho_t$ through the power-tempered ladder $\rho_{t,v}(d\theta)\propto g_{t-1}(\theta)^v\,\rho_{t-1}(d\theta),\ v\in[0,1].$ By \cref{lem:adjacent-renyi-reverse}, adjacent rungs are Rényi-warm when $\Delta=\Theta((RB)^{-1})$, so MALA mixes on each rung in $\tilde O\big(\sqrt d\,\kappa_t\log(K/\varepsilon_j)\big)$ oracle calls when rung $j$ is targeted to TV accuracy $\varepsilon_j$. Since the ladder has $K=\Theta(RB)$ rungs, the per-sample cost is $Q_t=\tilde O\!\big(\sqrt d\,\kappa_t\,K\log(K/\varepsilon_t)\big)$ when the fresh round-$t$ budget $\varepsilon_t$ is split uniformly across rungs, and \cref{prop:per-sample-bridged} shows that if $\mathrm{err}_{t-1}=d_{\mathrm{TV}}(\rho_{t-1},\tilde\rho_{t-1})$ denotes the inherited warm-start error, then $d_{\mathrm{TV}}(\rho_t,\tilde\rho_t)\le \mathrm{err}_{t-1}+\varepsilon_t$. Iterating this recursion gives the bound used in the regret term.
\end{proof}

\begin{restatable}{corollary}{CorollaryRBandCOST}
\label{cor:anytime-cewa}
Under the assumptions of \cref{th: regret guarantee of COMPUTABLE EWA}, choose
\[
\alpha_t=\frac{1}{2n},
\
\varepsilon_t=\frac{1}{20n^3},
\
\delta_t=\frac{\delta}{n},
\
s_t=\left\lceil 32n^3\log\left(\frac{n}{\delta}\right) \right\rceil.
\]
Then, with probability at least $1-\delta$, the estimator in \cref{eq: computable EWA estimators} satisfies
\( R_n \le O(d\log(Bn)) \)
with a total computational cost of order
\( \tilde{O}\big(B^3\,n^5\big). \)
\end{restatable}

\textbf{Remark.} The $\tilde{O}(B^{3}n^{5})$ complexity is a worst-case theoretical bound obtained under the choices of $(\alpha_t,\epsilon_t,s_t)$ needed for the analysis. In practice, as illustrated in the experiments of \cref{sec:experiments}, a single MALA chain with a constant number of samples per round is both computationally faster and performs well.

\section{Geometric limit and margin-based regret}
\label{sec:high_sep_regime}

Assuming bounded comparator norms is unnatural in some regimes: on linearly separable data, the optimal parameter $\theta^\star$ can have arbitrarily large norm (cf. \cite{qian2025refinedriskboundsunbounded}). Yet worst-case regret bounds for online logistic regression (\cref{tab:algorithms}) worsen with the prior scale $B$, encouraging conservative choices that may exclude large (or unbounded) $\|\theta^\star\|$. This raises some natural questions: \emph{how does EW (with isotropic Gaussian prior $\pi_0(\theta)\sim\mathcal{N}(0,B^2I_d)$) behave as $B$ grows and the data are linearly separable? Can the regret, in such a regime, become independent of $B$?} We show that, on separable data, EW admits a well-defined $B\to\infty$ limit and enjoys non-asymptotic bounds in terms of the margin $\gamma$, paralleling classical margin analyses \citep{freund1998large}. To make the dependence on the prior scale $B$ explicit throughout this section, we write in the following sections $\rho_t^B$ and $\hat p_t^B$ for the EW posterior and prediction respectively.

\subsection{$B\to\infty$, linear separability, and max-margin geometry.}
\label{subsec: Large B limit}
Assume the data are linearly separable, and define \(H_t \coloneqq \{\theta : y_i\langle x_i,\theta\rangle > 0,\ \forall i<t\}\) as the open version cone induced by the past data. The next propositions show that, as \(B\to\infty\), the EW predictive distribution converges to a directional majority vote under a standard Gaussian truncated to \(H_t\). In addition, on any fixed-margin slice, the corresponding truncated Gaussian has a unique mode whose direction coincides with the hard-margin SVM solution.

\begin{restatable}{proposition}{limitingEWA}
\label{prop:ewa-limit}
Fix $t\ge 2$ and assume the data $\{(x_i,y_i)\}_{i<t}$ are strictly linearly separable, meaning
\(H_t\neq\emptyset\). Then, for $(x,y)\in\R^d\times\{\pm1\}$ the EW prediction in \cref{eq: estimators EWA} has a limit as $B\to\infty$ given by
\begin{equation}
\begin{split}
\label{eq: limiting predictions of EWA}
\lim_{B\to\infty}\hat p_t^B(x,y)
=
\mathbb{P}&_{\theta\sim\rho_t^\infty} \big(y\langle x,\theta\rangle>0\big) +\frac12\,\mathbb{P}_{\theta\sim\rho_t^\infty} \big(y\langle x,\theta\rangle=0\big)
\eqqcolon \hat{p}_t^\infty(x,y),
\end{split}
\end{equation}
where the limiting distribution $\rho_t^\infty$ is the standard Gaussian truncated to $H_t$:
\[
\rho_t^\infty(d\theta)
:=
\frac{e^{-\|\theta\|^2/2}\mathbf 1\{\theta\in H_t\}\,d\theta}
{\int_{\R^d} e^{-\|u\|^2/2}\mathbf 1\{u\in H_t\}\,du}.
\]
\end{restatable}

\begin{proof}[Sketch of proof]
After the rescaling \(\tilde\theta=\theta/B\), the prior becomes standard Gaussian and the likelihood terms \(\sigma\big(B y_i\langle x_i,\tilde\theta\rangle\big)\) converge pointwise to \( \mathbf 1\{y_i\langle x_i,\tilde\theta\rangle>0\} + \tfrac12 \mathbf 1\{y_i\langle x_i,\tilde\theta\rangle=0\}\) as \(B\to\infty\). Under strict separability (\(H_t\neq\emptyset\)), dominated convergence allows us to pass the limit inside the integrals, which yields the limiting predictive law \(\hat p_t^\infty\) and shows that the (rescaled) posterior converges to the standard Gaussian conditioned on \(H_t\).
\end{proof}

Because the standard Gaussian is rotationally invariant and the constraints $y_i\langle x_i,\theta\rangle>0$ depend only on the direction of $\theta$, such a direction under $\rho_t^\infty$ is then uniform on the spherical polyhedral cone $H_t\cap\mathbb S^{d-1}$, independently of its radius. As a consequence, for every nonzero query $x$, the limiting predictive probability in \eqref{eq: limiting predictions of EWA} is a \emph{solid--angle vote}: it is exactly the fraction (with respect to surface measure) of directions in the version cone that would label $(x,y)$ correctly. Indeed, when $x\neq 0$, the boundary set $\{\theta:\langle x,\theta\rangle=0\}$ is a hyperplane and therefore has zero $\rho_t^\infty$-mass, so the additional tie term vanishes. Thus, while the MLE diverges in separable logistic regression, the Exponential Weights predictor remains well defined as $B\to\infty$ and, for nondegenerate queries, depends only on the solid angle of $H_t$ and the relative position of $x$, implying that EW also adapts to the geometry of data. Furthermore, we now show how such a limiting predictor connects with the hard-margin SVM solution, as also highlighted by the toy 2-D experiment of \cref{fig:voter-svm}.

\begin{wrapfigure}{r}{0.42\textwidth}
  \vspace{-0.8em}
  \centering
  \resizebox{0.45\textwidth}{!}{%
    \input{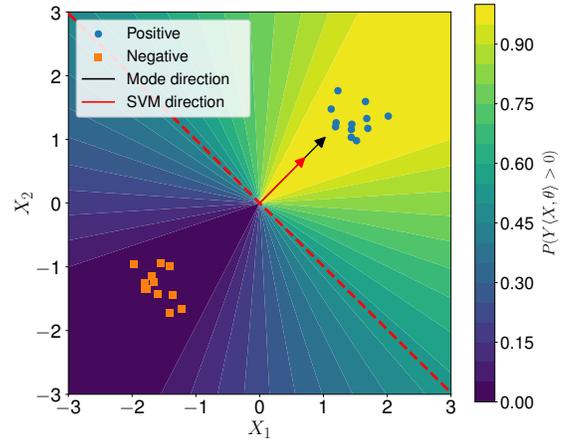}%
  }
  \caption{2-D example of solid-angle voter's positive class probability assignments on a margin slice with fixed $\gamma$.}
  \label{fig:voter-svm}
  \vspace{-0.8em}
\end{wrapfigure}

\begin{restatable}{proposition}{EWAsvm}
\label{prop:mode-svm}
Fix $t\ge 2$ and assume strict separability, i.e. $H_t\neq\emptyset$. For any margin $\gamma>0$, define the closed margin slice
\[
S_{t,\gamma}\ :=\ \{\theta\in\R^d:\ y_i\langle x_i,\theta\rangle\ge \gamma,\ \forall i<t\},
\]
and let $\rho_{t,\gamma}^\infty$ be the standard Gaussian truncated to $S_{t,\gamma}$
\[
\rho_{t,\gamma}^\infty(d\theta)
\ :=\
\frac{e^{-\|\theta\|^2/2}\,\mathbf 1\{\theta\in S_{t,\gamma}\}\,d\theta}
{\int_{\R^d}e^{-\|u\|^2/2}\,\mathbf 1\{u\in S_{t,\gamma}\}\,du}.
\]
Then, $\rho_{t,\gamma}^\infty$ has a \emph{unique mode} $\theta_{t,\gamma}\in S_{t,\gamma}$. Moreover, let $w_{t,\mathrm{svm}}$ denote the hard-margin SVM solution; then, the mode $\theta_{t,\gamma}$ satisfies $\theta_{t,\gamma}\ =\ \gamma\, w_{t,\mathrm{svm}}$. In particular, the mode direction $\theta_{t,\gamma}/\|\theta_{t,\gamma}\|$ is independent of $\gamma$
and equals $w_{t,\mathrm{svm}}/\|w_{t,\mathrm{svm}}\|$.
\end{restatable}

\begin{proof}[Sketch of proof]
The proof follows by noting that $\theta\in S_{t,\gamma}\Leftrightarrow \theta=\gamma w$ with $w\in S_{t,1}$ and that $S_{t,1}$ is nonempty (see Lemma~\ref{lem:Scalings slice margin} in Appendix). Moreover, maximizing the truncated density of $\rho_{t,\gamma}^\infty$ over $S_{t,\gamma}$ is equivalent to minimizing $\|\theta\|^2$ over $S_{t,\gamma}$; after the rescaling $\theta=\gamma w$ this becomes exactly the hard--margin SVM problem.
\end{proof}

\begin{restatable}{corollary}{ModeLimitSmallMargin}
\label{corollary:mode-svm_limit}
Under the assumptions of \cref{prop:mode-svm}, let $\rho_t^\infty$ and $\rho_{t,\gamma}^\infty$ be the standard Gaussians truncated to $H_t$ and $S_{t,\gamma}$ respectively.
Then: (i) as $\gamma\rightarrow 0$, $\rho_{t,\gamma}^\infty$ converges to $\rho_t^\infty$ in total variation; (ii) For every $\gamma>0$, the measure
$\rho_{t,\gamma}^\infty$ has a unique mode $\theta_{t,\gamma} =\gamma\,w_{t,\mathrm{svm}}$,
with all modes lying on the ray $\{\lambda w_{t,\mathrm{svm}}:\lambda>0\}$.
\end{restatable}

These results provide a Bayesian perspective closely related to the implicit bias of gradient methods for logistic regression on separable data. Indeed, the optimization literature shows that gradient descent on the logistic loss drives $\|w_t\|\to\infty$ while its \emph{direction} converges to the hard-margin SVM direction, namely $\frac{w_t}{\|w_t\|}\to \frac{w_{\mathrm{svm}}}{\|w_{\mathrm{svm}}\|}$ (see, e.g., \citet{Soudry18,ji2018risk,wu2023implicit,zhang2025minimax}). In contrast, our analysis is probabilistic: as $B\to\infty$ under strict separability, the posterior converges to a standard Gaussian truncated to the version cone $H_t$, and on any fixed-margin slice $S_{t,\gamma}$ its mode is aligned with the hard-margin SVM direction $w_{t,\mathrm{svm}}$. This suggests a complementary geometric mechanism by which a Bayesian aggregator recovers the same max-margin direction in the separable regime.

\textbf{Remark.} In our current setup we incorporate the bias via the standard dummy-feature augmentation
$\tilde x=(x,1)$ and an isotropic Gaussian prior $\tilde\theta=(w,b)\sim\mathcal N(0,B^2 I)$. Consequently, the induced regularizer is $\|w\|^2+b^2$, so both our regret analysis and the large-$B$ geometric limit correspond to an SVM variant in which the intercept $b$ is \emph{regularized} as well, rather than the classical affine SVM where $b$ is left unpenalized.

\subsection{Large-$B$ regret and margin geometry}
\label{subsec: large B regret non-asymptotic}
The geometric convergence of EW predictions to a max-margin voter leads to the following question: \emph{can the above described geometry of EW in the linearly separable case be translated into a non-asymptotic regret guarantee that only depends on the margin $\gamma$ rather than on $B$?} 
We answer this affirmatively; in particular, once $B$ is sufficiently large, the regret becomes independent of $B$ and is controlled by margin-related quantities as shown in the following result.

\begin{restatable}{theorem}{geometricEWAallBsimple}
\label{thm:ewa-interpolation-simple}
Assume the data are linearly separable with margin $\gamma>0$, $\|x_t\|\le R$, and $d\geq2$. In addition, let $\bar\gamma\coloneqq\gamma/R$ and define 
$$L_n^B \coloneqq \sum_{t=1}^n -\log \hat p_t^B(x_t,y_t)$$
as the cumulative predictive loss of EW with Gaussian prior $\pi_0=\mathcal N(0,B^2I_d)$. Then, for every $B>\frac{2(2+\sqrt{2})\log(2n)}{\gamma\sqrt{d-1}} \eqqcolon B_\mathrm{critic}$,
it holds that
\[
L_n^B \le (d-1)\log\left(\frac{2}{\bar{\gamma}}\right) + c_1 \log d,
\]
with $c_1>0$ being an absolute constant.
\end{restatable}

\begin{proof}[Sketch of proof]
    Write the EW cumulative predictive loss as $e^{-L_n^B}=\int \prod_{i\le n}\sigma(y_i\langle x_i,\theta\rangle)\,\pi_0(d\theta)$. Under separability, fix a unit separator $u$ with margin $\gamma$ and consider a \emph{cone of good parameters}: directions $v$ in a spherical cap around $u$ and radii $\lambda$ large enough.
    A simple geometry inequality shows that for every $\theta=\lambda v$ in this cone the margin is uniformly large, $\min_i y_i\langle x_i,\theta\rangle \gtrsim \lambda \,\alpha(\bar\gamma)$, so choosing $\lambda$ so that this is at least $m\coloneqq\log(2n)$ makes each likelihood term close to $1$ and hence $\prod_i \sigma(y_i\langle x_i,\theta\rangle)\ge e^{-1}$. Restricting the integral to this cone gives $e^{-L_n^B}\ge e^{-1}\pi_0(\text{cone})$. Finally, because $\pi_0=\mathcal N(0,B^2I)$ factorizes into \emph{angle} (uniform on $S^{d-1}$) and \emph{radius} ($\chi_d$), $\pi_0(\text{cone})$ splits into a cap probability and a radial tail: the cap term yields $(d-1)\log(2/\bar\gamma)+O(\log d)$, while for $B\ge B_{\mathrm{crit}}$ the radial tail is a constant (at least $1/2$), making the bound independent of $B$ beyond this threshold.
\end{proof}

\textbf{Remark.} The assumption $d\ge 2$ is only used to invoke the spherical-cap bound and the threshold
$B_{\mathrm{critic}}\asymp \log(2n)/(\gamma\sqrt{d-1})$. When $d=1$, the angular term disappears since
$\mathbb S^0=\{-1,+1\}$, so the proof reduces to controlling only the radial Gaussian tail. In that case,
the cumulative loss becomes independent of $B$ once $B\ge k\,\log(2n)/\gamma$ for some absolute constant $k>0$.

In other words, once the prior scale $B$ exceeds the margin-dependent threshold $B_{\mathrm{critic}}$, the cumulative predictive loss of EW becomes independent of $B$: the large-$B$ regime is governed by the angular geometry of the separating directions, rather than by the radial scale of the Gaussian prior. In the separable case, this immediately yields the same control for the regret, since
\[
R_n^B
=
L_n^B-\inf_{\|\theta\|\le B}\sum_{t=1}^n -\log \sigma \big(y_t\langle x_t,\theta\rangle\big)
\le L_n^B,
\]
because the comparator loss is nonnegative. Thus, although the best linear comparator drives its logistic loss to zero in the separable regime, EW itself remains uniformly controlled once $B\ge B_{\mathrm{critic}}$.

\paragraph{Consistency with the lower bound of \citet{foster2018logistic}.}
More generally, we can prove a margin-dependent upper bound on the regret that is valid for all $B>0$. In particular, for a generic $B$,
\begin{equation}
\label{eq: regret consistency LB}
R_n \le (d-1)\log\!\left(\tfrac{2}{\bar{\gamma}}\right) + \tfrac{1}{2}\left(\tfrac{\lambda_0}{B}\right)^2 + c_2 d\log d,
\end{equation}
where $\lambda_0$ depends on the cap opening $t_1$ (see \cref{thm:ewa-interpolation-simple-threshold} in the appendix) and $c_2>0$ is an absolute constant.
This is consistent with the lower-bound construction of \citet[Theorem 2]{foster2018logistic}, where the authors consider a separable instance with margin $\gamma(B)=\log(n)/B$. Plugging this $B$-dependent margin into \eqref{eq: regret consistency LB} recovers the same logarithmic dependence on $B$, namely through the term $\log(1/\gamma(B))=\log(B/\log n)$, while the radial term $(\lambda_0/B)^2$ remains bounded. Thus, for fixed data and sufficiently large $B$, the rate no longer exhibits a radial dependence on $B$; however, when the adversary changes the geometry by shrinking $\gamma$ with $B$, a $\log B$ term necessarily reappears through $\log(1/\gamma)$, in agreement with the lower-bound construction.

\textbf{Remark.} These results are statistical in nature. Computationally, the complexity still increases with the prior scale $B$, since our sampler targets a posterior whose condition number $\kappa_t$ grows with $B$. In our efficient implementation, the total cost is $\tilde O(\sqrt d\, B^3 R^3 n^5)$, up to polylogarithmic factors, so larger $B$ remains more expensive, as in other methods in the literature. Designing an implementation with $B$-robust complexity is an interesting direction for future work.

Observe that the result of \cref{thm:ewa-interpolation-simple} holds for arbitrary data sequences, provided they are linearly separable. In what follows, we specialize it to a benign i.i.d. setting, where it yields a regret bound of order $O(d \log(n))$.

\paragraph{Example: well-specified logistic regression with Gaussian design.}
Let us consider the following well-specified logistic model with true parameter $\theta^\star$, where $\|\theta^\star\|=B$ and $d\geq 2$ \footnote{We could treat separately the case of $d=1$ as we did for \cref{thm:ewa-interpolation-simple}.}. For all $t=1,...,n$, we set $x_t\iid N(0,I_d)$ and the labels $y_t\in\{\pm 1\}$ are such that
\[
\mathbb{P}(y_t=1\mid x_t)=\sigma(\langle x_t,\theta^\star\rangle),\qquad \sigma(t)=\frac{1}{1+e^{-t}}.
\]

\begin{restatable}{proposition}{EWArblargeBiidUnified}
\label{cor:ewa-largeB-delta-unified}
Consider the above-defined logistic model. Fix confidence parameter $\delta\in(0,1)$ and assume
\begin{equation}
\label{eq:unified-B}
B \;\ge\; \max\!\left\{
\frac{6n}{\sqrt{2\pi}\,\delta}\;,\;
\frac{12(2+\sqrt{2})\,n\,\log(2n)}{\delta\,\sqrt{\,d-1\,}}
\right\}.
\end{equation}
Let $u\coloneqq \theta^\star/\|\theta^\star\|$; then, with probability at least $1-\delta$, the dataset is linearly separable and 
\begin{equation*}
\gamma \coloneqq \min_{1\le t\le n} y_t\langle u,x_t\rangle \ge \frac{\delta}{6n}, \quad R\coloneqq \max_{1\le t\le n}\|x_t\|  \le \sqrt d+\sqrt{2\log\frac{3n}{\delta}}\,,
\end{equation*}
In particular, Theorem~\ref{thm:ewa-interpolation-simple} holds and
\begin{equation*}
\label{eq:ewa-delta-final-unified}
R_n\ \le \ (d-1)\,\log\left(
\frac{12nR}{\delta}
\right) +  c_1\log(d).
\end{equation*}
\end{restatable}

\begin{proof}[Sketch of proof]
With probability at least $1-\delta$, we have (i) noiseless labels aligned with $u$ (event \(\mathcal E_1\)), (ii) empirical margin \(\gamma\ge \delta/(6n)\) (event \(\mathcal E_2\)), and (iii) radius \(R\le \sqrt d+\sqrt{2\log(3n/\delta)}\) (event \(\mathcal E_3\)).
Furthermore, \cref{thm:ewa-interpolation-simple} holds because of \eqref{eq:unified-B}, so \(L_n^B\le (d-1)\log(2/\bar\gamma)+O(\log(d))\).
Substituting the bounds for \(\gamma\) and \(R\) on \(\mathcal E_1\cap\mathcal E_2\cap\mathcal E_3\) gives the stated expression; in particular \(R_n\le L_n^B\) in the separable case.
\end{proof}

\section{Experiments}
\label{sec:experiments}

\begin{figure*}[t]
  \centering
  \includegraphics[width=\textwidth]{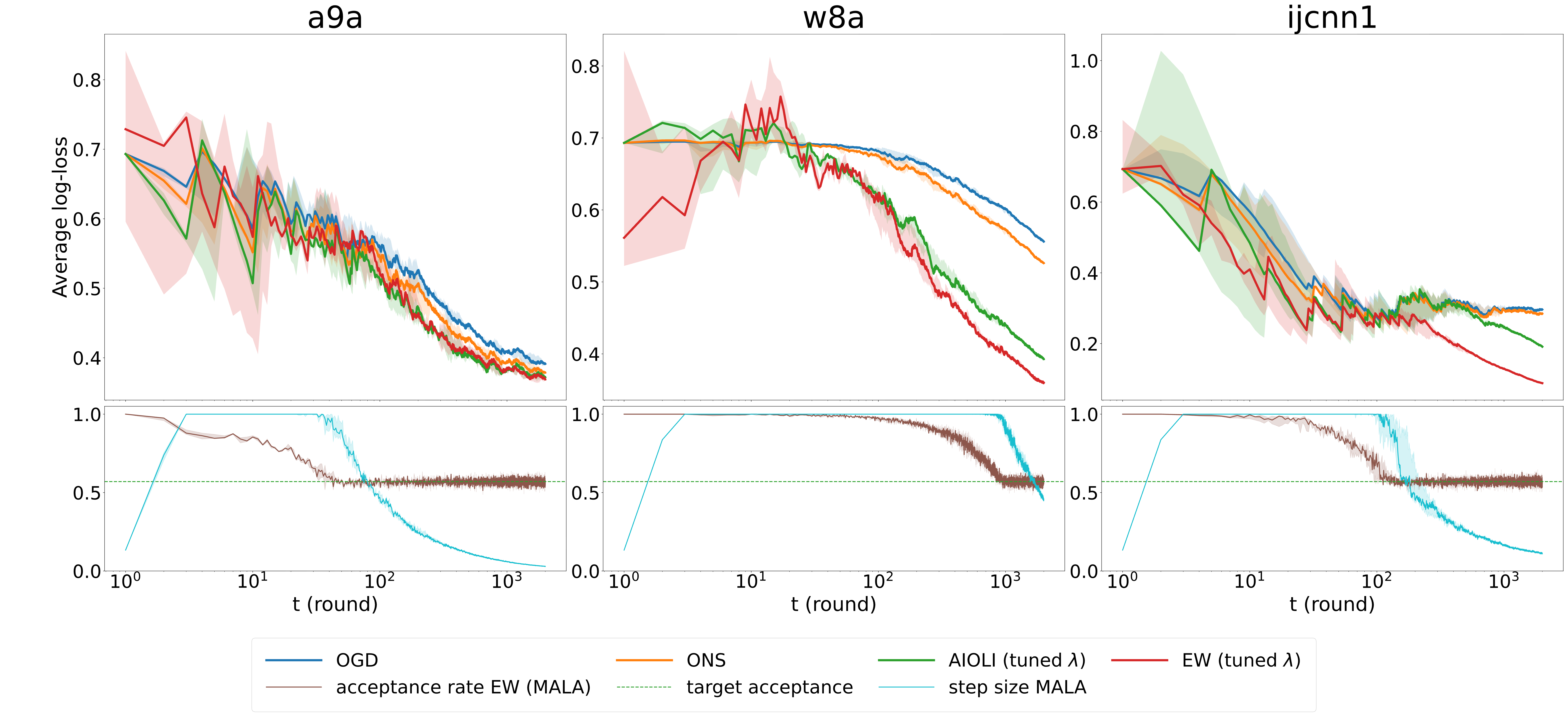} 
  \caption{\textbf{(Top):} Median (across $5$ repeats in which we modify the sequential order of the data) average log-loss with shaded interquartile range (25th--75th percentile). \textbf{(Bottom):} Acceptance rate and step size of MALA (used for EW) which is implemented to get a target acceptance rate of $0.57$. Both plots are in log-x scale.}
  \label{fig:performance_alg}
\end{figure*}

\begin{figure*}[t]
  \centering
  \includegraphics[width=\textwidth]{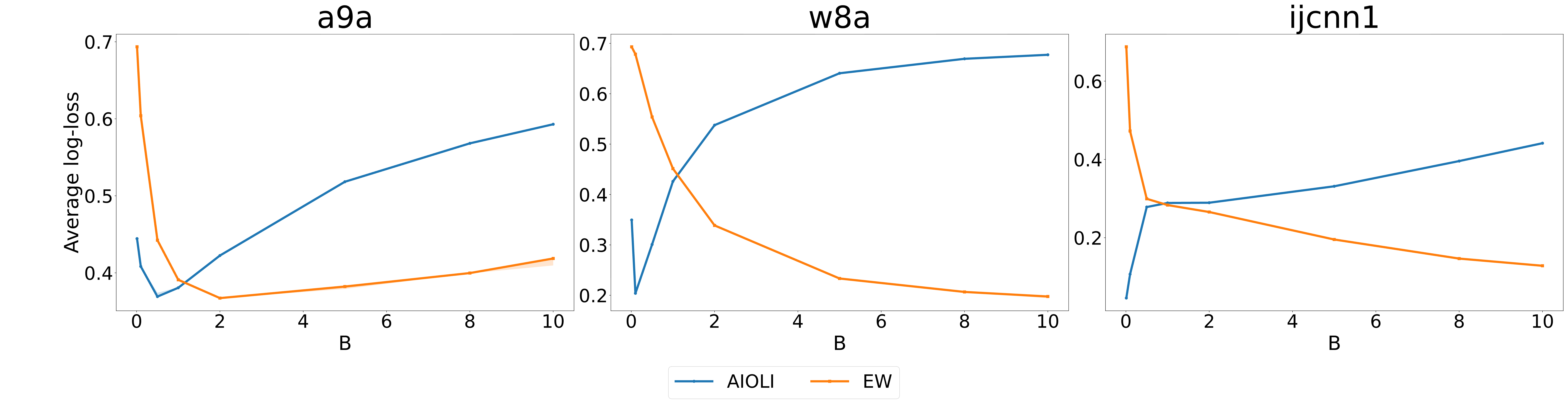} 
  \caption{Average log-loss at round \(n=1000\) versus different values of $B$ for AIOLI (\(\lambda=1/B^2\)) and EW (Gaussian prior scale \(B\)). Curves show the median over 5 random permutations of the sequential data with shaded regions are interquartile ranges (IQR, 25th--75th percentiles).}
  \label{fig:b_sweep_alg}
\end{figure*}

We use three LIBSVM\footnote{\texttt{https://www.csie.ntu.edu.tw/~cjlin/libsvm/}} binary classification datasets: \texttt{a9a}, \texttt{w8a}, and \texttt{ijcnn1}, where we keep the first $n=2000$ samples and make them \emph{sequential}. At each round $t$, each method predicts the logit score $\hat{z}_t$ for $x_t$ using only history up to $t-1$, then receives the true $y_t$, incurs the logistic loss $\ell_\mathrm{log}(\hat{z}_t,y_t)$, and updates \footnote{Data and code available \href{https://github.com/FedericoDiGennaro/Efficient-LogReg-with-Mixture-of-Sigmoids}{here \faGithub}.}.

\paragraph{Experiment 1.} First of all, we evaluate the performance of our EW with isotropic Gaussian prior by measuring the average log-loss
\[
\bar L_t = \frac{1}{t}\sum_{s=1}^t \ell_\mathrm{log}(\hat{z}_s,y_s)
\]
together with three baseline algorithms for online (binary) logistic regression: OGD, ONS, AIOLI (cf. \cref{tab:algorithms}). For both AIOLI and EW we tuned (grid-search) the parameter $\lambda$ (namely the regularization parameter in AIOLI and the variance of the Gaussian prior in EW) regardless of the tuning suggestions provided by the worst-case theoretical results for both of them. From \cref{fig:performance_alg} we can observe that these two algorithms have comparable (and always better compared to OGD and ONS) performance across the three datasets, with EW that tends to be slightly preferable. Notice that each experiment is repeated $5$ times.
For each repeat, we draw a fresh random permutation of the train split and run the full online process.

\paragraph{Experiment 2.} We then fix $\lambda$ to be $B$-dependent as the theory suggests for both AIOLI and EW with isotropic Gaussian prior. Then, we plot again the average logistic loss but this time only at time $n=1000$ and compared against different values of $B$. As expected, AIOLI is preferred once $B$ is very low (in that case EW is \emph{constrained} too much by the prior) but, as soon as $B$ starts to increase, EW is preferable. However, getting the best-of-both-worlds is left as an open question.

\paragraph{Remark on EW computation.} Although the parameter $s_t$ (number of independent MALA chains) we set to have a provable computational algorithm which does not affect the regret rate was $t$-dependent, in practice it turns out we can use a fixed (and constant) number of MALA chains $s$ (slightly adapted as $B$ increases in the case of \cref{fig:b_sweep_alg}). This makes the implementation of EW of order $O(n^2)$, far cheaper than the worst case complexity we derived for our provably regret-optimal computable EW in \cref{eq: computable EWA estimators}.

\section{Conclusions}
We studied Exponential Weights (EW) for online logistic regression, proved a worst-case regret bound of \(\cO\big(d\log(Bn)\big)\), and analyzed a practical MALA-based implementation that makes the method computationally viable. In the linearly separable regime, we showed that EW adapts to data geometry and that, in the large-\(B\) limit, its predictions admit a solid-angle interpretation connected to the hard-margin SVM. Promising future directions include: (i) removing or substantially reducing the dependence on \(B\) in the sampling step, for instance via preconditioning or geometry-aware proposals; (ii) extending the approach to the multiclass setting with \(K>2\), in the spirit of \citet{foster2018logistic}; and (iii) deriving high-probability guarantees using techniques such as those in \citep{van2023high}.

\section*{Acknowledgments}
Federico Di Gennaro gratefully acknowledges support from the Hasler-Stiftung Foundation for his visit to UC Berkeley, where this work was partially carried out.

\clearpage

\bibliographystyle{apalike}
\bibliography{references}

\begin{thebibliography}{}

\bibitem[Agarwal et~al., 2022]{agarwal2022efficient}
Agarwal, N., Kale, S., and Zimmert, J. (2022).
\newblock Efficient methods for online multiclass logistic regression.
\newblock In {\em International Conference on Algorithmic Learning Theory}, pages 3--33. PMLR.

\bibitem[Albert and Anderson, 1984]{albert1984existence}
Albert, A. and Anderson, J.~A. (1984).
\newblock On the existence of maximum likelihood estimates in logistic regression models.
\newblock {\em Biometrika}, 71(1):1--10.

\bibitem[Altschuler and Chewi, 2024]{altschuler2024faster}
Altschuler, J.~M. and Chewi, S. (2024).
\newblock Faster high-accuracy log-concave sampling via algorithmic warm starts.
\newblock {\em Journal of the ACM}, 71(3):1--55.

\bibitem[Auer et~al., 1995]{auer1995gambling}
Auer, P., Cesa-Bianchi, N., Freund, Y., and Schapire, R.~E. (1995).
\newblock Gambling in a rigged casino: The adversarial multi-armed bandit problem.
\newblock In {\em Proceedings of IEEE 36th annual foundations of computer science}, pages 322--331. IEEE.

\bibitem[Auer et~al., 2002]{auer2002nonstochastic}
Auer, P., Cesa-Bianchi, N., Freund, Y., and Schapire, R.~E. (2002).
\newblock The nonstochastic multiarmed bandit problem.
\newblock {\em SIAM journal on computing}, 32(1):48--77.

\bibitem[Bubeck et~al., 2012]{bubeck2012towards}
Bubeck, S., Cesa-Bianchi, N., and Kakade, S.~M. (2012).
\newblock Towards minimax policies for online linear optimization with bandit feedback.
\newblock In {\em Conference on Learning Theory}, pages 41--1. JMLR Workshop and Conference Proceedings.

\bibitem[Bubeck et~al., 2018]{bubeck2018sampling}
Bubeck, S., Eldan, R., and Lehec, J. (2018).
\newblock Sampling from a log-concave distribution with projected langevin monte carlo.
\newblock {\em Discrete \& Computational Geometry}, 59:757--783.

\bibitem[Cesa-Bianchi et~al., 2004]{cesa2004generalization}
Cesa-Bianchi, N., Conconi, A., and Gentile, C. (2004).
\newblock On the generalization ability of on-line learning algorithms.
\newblock {\em IEEE Transactions on Information Theory}, 50(9):2050--2057.

\bibitem[Cesa-Bianchi and Lugosi, 2006]{cesa2006prediction}
Cesa-Bianchi, N. and Lugosi, G. (2006).
\newblock {\em Prediction, learning, and games}.
\newblock Cambridge university press.

\bibitem[Chewi, 2024]{ChewiBook24}
Chewi, S. (2024).
\newblock {\em Log-Concave Sampling}.
\newblock Unfinished draft.

\bibitem[Chewi et~al., 2021]{chewi2021optimal}
Chewi, S., Lu, C., Ahn, K., Cheng, X., Le~Gouic, T., and Rigollet, P. (2021).
\newblock Optimal dimension dependence of the metropolis-adjusted langevin algorithm.
\newblock In {\em Conference on Learning Theory}, pages 1260--1300. PMLR.

\bibitem[Cutkosky, 2019]{pmlr-v97-cutkosky19a}
Cutkosky, A. (2019).
\newblock Anytime online-to-batch, optimism and acceleration.
\newblock In Chaudhuri, K. and Salakhutdinov, R., editors, {\em Proceedings of the 36th International Conference on Machine Learning}, volume~97 of {\em Proceedings of Machine Learning Research}, pages 1446--1454. PMLR.

\bibitem[Del~Moral et~al., 2006]{del2006sequential}
Del~Moral, P., Doucet, A., and Jasra, A. (2006).
\newblock Sequential monte carlo samplers.
\newblock {\em Journal of the Royal Statistical Society Series B: Statistical Methodology}, 68(3):411--436.

\bibitem[{Di Gennaro} et~al., 2025]{digennaro2026instancedependentregretboundsnonstochastic}
{Di Gennaro}, F., Eldowa, K., and Cesa-Bianchi, N. (2025).
\newblock Instance-dependent regret bounds for nonstochastic linear partial monitoring.
\newblock In {\em The Thirty-ninth Annual Conference on Neural Information Processing Systems}.

\bibitem[Donsker and Varadhan, 1975]{donsker1975asymptotic}
Donsker, M.~D. and Varadhan, S.~S. (1975).
\newblock Asymptotic evaluation of certain markov process expectations for large time, i.
\newblock {\em Communications on pure and applied mathematics}, 28(1):1--47.

\bibitem[Doodson, 1917]{bmkmedianGamma}
Doodson, A.~T. (1917).
\newblock Iii. relation of the mode, median and mean in frequency curves.
\newblock {\em Biometrika}, 11(4):425--429.

\bibitem[Dwivedi et~al., 2019]{metropolis_fast}
Dwivedi, R., Chen, Y., Wainwright, M.~J., and Yu, B. (2019).
\newblock Log-concave sampling: Metropolis-hastings algorithms are fast.
\newblock {\em Journal of Machine Learning Research}, 20(183):1--42.

\bibitem[Foster et~al., 2018]{foster2018logistic}
Foster, D.~J., Kale, S., Luo, H., Mohri, M., and Sridharan, K. (2018).
\newblock Logistic regression: The importance of being improper.
\newblock In {\em Conference on learning theory}, pages 167--208. PMLR.

\bibitem[Freund and Schapire, 1997]{freund1997decision}
Freund, Y. and Schapire, R.~E. (1997).
\newblock A decision-theoretic generalization of on-line learning and an application to boosting.
\newblock {\em Journal of computer and system sciences}, 55(1):119--139.

\bibitem[Freund and Schapire, 1998]{freund1998large}
Freund, Y. and Schapire, R.~E. (1998).
\newblock Large margin classification using the perceptron algorithm.
\newblock In {\em Proceedings of the eleventh annual conference on Computational learning theory}, pages 209--217.

\bibitem[Hastie et~al., 2009]{HastieTibshiraniFriedman2009}
Hastie, T., Tibshirani, R., and Friedman, J. (2009).
\newblock {\em The Elements of Statistical Learning}.
\newblock Springer, 2 edition.

\bibitem[Hazan, 2016]{Hazan2016}
Hazan, E. (2016).
\newblock {\em Introduction to Online Convex Optimization}.
\newblock Foundations and Trends in Optimization.

\bibitem[Hazan et~al., 2007]{hazan2007logarithmic}
Hazan, E., Agarwal, A., and Kale, S. (2007).
\newblock Logarithmic regret algorithms for online convex optimization.
\newblock {\em Machine Learning}, 69(2):169--192.

\bibitem[Hazan et~al., 2014]{hazan2014logistic}
Hazan, E., Koren, T., and Levy, K.~Y. (2014).
\newblock Logistic regression: Tight bounds for stochastic and online optimization.
\newblock In {\em Conference on Learning Theory}, pages 197--209. PMLR.

\bibitem[J{\'e}z{\'e}quel et~al., 2020]{rudi_binary_efficient_logreg}
J{\'e}z{\'e}quel, R., Gaillard, P., and Rudi, A. (2020).
\newblock Efficient improper learning for online logistic regression.
\newblock In Abernethy, J. and Agarwal, S., editors, {\em Proceedings of Thirty Third Conference on Learning Theory}, volume 125 of {\em Proceedings of Machine Learning Research}, pages 2085--2108. PMLR.

\bibitem[J{\'e}z{\'e}quel et~al., 2021]{jezequel2021mixability}
J{\'e}z{\'e}quel, R., Gaillard, P., and Rudi, A. (2021).
\newblock Mixability made efficient: Fast online multiclass logistic regression.
\newblock {\em Advances in Neural Information Processing Systems}, 34:23692--23702.

\bibitem[Ji and Telgarsky, 2018]{ji2018risk}
Ji, Z. and Telgarsky, M. (2018).
\newblock Risk and parameter convergence of logistic regression.
\newblock {\em arXiv preprint arXiv:1803.07300}.

\bibitem[Kakade and Ng, 2005]{kakade_ng_bayesianalg}
Kakade, S.~M. and Ng, A. (2005).
\newblock Online bounds for {B}ayesian algorithms.
\newblock In Saul, L., Weiss, Y., and Bottou, L., editors, {\em Advances in Neural Information Processing Systems}, volume~17. MIT Press.

\bibitem[Lattimore and Szepesv{\'a}ri, 2019]{lattimore2019cleaning}
Lattimore, T. and Szepesv{\'a}ri, C. (2019).
\newblock Cleaning up the neighborhood: A full classification for adversarial partial monitoring.
\newblock In {\em Algorithmic Learning Theory}, pages 529--556. PMLR.

\bibitem[Littlestone and Warmuth, 1994]{littlestone1994weighted}
Littlestone, N. and Warmuth, M.~K. (1994).
\newblock The weighted majority algorithm.
\newblock {\em Information and computation}, 108(2):212--261.

\bibitem[McCullagh and Nelder, 1989]{McCullaghNelder1989}
McCullagh, P. and Nelder, J.~A. (1989).
\newblock {\em Generalized Linear Models}.
\newblock Chapman and Hall, 2 edition.

\bibitem[McMahan and Streeter, 2009]{mcmahan2009tighter}
McMahan, H.~B. and Streeter, M.~J. (2009).
\newblock Tighter bounds for multi-armed bandits with expert advice.
\newblock In {\em COLT}.

\bibitem[Mourtada and Ga{\"\i}ffas, 2022]{mourtada2022improper}
Mourtada, J. and Ga{\"\i}ffas, S. (2022).
\newblock An improper estimator with optimal excess risk in misspecified density estimation and logistic regression.
\newblock {\em Journal of Machine Learning Research}, 23(31):1--49.

\bibitem[Murphy, 2012]{Murphy2012}
Murphy, K.~P. (2012).
\newblock {\em Machine Learning: A Probabilistic Perspective}.
\newblock MIT Press.

\bibitem[Neal, 2001]{neal2001annealed}
Neal, R.~M. (2001).
\newblock Annealed importance sampling.
\newblock {\em Statistics and computing}, 11(2):125--139.

\bibitem[Qian et~al., 2025]{qian2025refinedriskboundsunbounded}
Qian, J., Rakhlin, A., and Zhivotovskiy, N. (2025).
\newblock Refined risk bounds for unbounded losses via transductive priors.

\bibitem[Roberts and Stramer, 2002]{roberts2002langevin}
Roberts, G.~O. and Stramer, O. (2002).
\newblock Langevin diffusions and metropolis-hastings algorithms.
\newblock {\em Methodology and computing in applied probability}, 4(4):337--357.

\bibitem[Shalev-Shwartz, 2012]{ShalevShwartz2012}
Shalev-Shwartz, S. (2012).
\newblock Online learning and online convex optimization.
\newblock {\em Foundations and Trends in Machine Learning}, 4(2):107--194.

\bibitem[Shamir, 2020]{shamir2020logistic}
Shamir, G.~I. (2020).
\newblock Logistic regression regret: What’s the catch?
\newblock In {\em Conference on Learning Theory}, pages 3296--3319. PMLR.

\bibitem[Soudry et~al., 2018]{Soudry18}
Soudry, D., Hoffer, E., Nacson, M.~S., Gunasekar, S., and Srebro, N. (2018).
\newblock The implicit bias of gradient descent on separable data.
\newblock {\em Journal of Machine Learning Research}, 19(70):1--57.

\bibitem[van~der Hoeven et~al., 2023]{van2023high}
van~der Hoeven, D., Zhivotovskiy, N., and Cesa-Bianchi, N. (2023).
\newblock High-probability risk bounds via sequential predictors.
\newblock {\em arXiv preprint arXiv:2308.07588}.

\bibitem[van Erven et~al., 2015]{JMLRvanerven15a}
van Erven, T., Gr{{\"u}}nwald, P.~D., Mehta, N.~A., Reid, M.~D., and Williamson, R.~C. (2015).
\newblock Fast rates in statistical and online learning.
\newblock {\em Journal of Machine Learning Research}, 16(54):1793--1861.

\bibitem[Vovk, 2001]{vovk2001competitive}
Vovk, V. (2001).
\newblock Competitive on-line statistics.
\newblock {\em International Statistical Review}, 69(2):213--248.

\bibitem[Vovk, 1990]{vovk1990aggregating}
Vovk, V.~G. (1990).
\newblock Aggregating strategies.
\newblock In Fulk, M. and Case, J., editors, {\em Proceedings of the Third Annual Workshop on Computational Learning Theory}, COLT ’90, pages 371--383, San Mateo, CA, USA. Morgan Kaufmann.

\bibitem[Wu et~al., 2022]{wu2022precise}
Wu, C., Heidari, M., Grama, A., and Szpankowski, W. (2022).
\newblock Precise regret bounds for log-loss via a truncated bayesian algorithm.
\newblock {\em Advances in Neural Information Processing Systems}, 35:26903--26914.

\bibitem[Wu et~al., 2023]{wu2023implicit}
Wu, J., Braverman, V., and Lee, J.~D. (2023).
\newblock Implicit bias of gradient descent for logistic regression at the edge of stability.
\newblock {\em Advances in Neural Information Processing Systems}, 36:74229--74256.

\bibitem[Zhang et~al., 2025]{zhang2025minimax}
Zhang, R., Wu, J., Lin, L., and Bartlett, P.~L. (2025).
\newblock Minimax optimal convergence of gradient descent in logistic regression via large and adaptive stepsizes.
\newblock {\em arXiv preprint arXiv:2504.04105}.

\bibitem[Zhang, 2004]{zhang2004solving}
Zhang, T. (2004).
\newblock Solving large scale linear prediction problems using stochastic gradient descent algorithms.
\newblock In {\em Proceedings of the twenty-first international conference on Machine learning}, page 116.

\bibitem[Zinkevich, 2003]{zinkevich2003online}
Zinkevich, M. (2003).
\newblock Online convex programming and generalized infinitesimal gradient ascent.
\newblock In {\em Proceedings of the 20th international conference on machine learning (icml-03)}, pages 928--936.

\end{thebibliography}

\onecolumn
\appendix

\section*{Appendix}

\section{Exponential Weights Algorithm}
\label{app-sec: EW}
The Exponential Weights (EW) Algorithm \citep{vovk1990aggregating}, also known as the multiplicative weights or Hedge algorithm, is a fundamental online learning method for sequential decision problems. It maintains a weight distribution over a set of actions or experts and updates these weights multiplicatively in response to observed losses. Concretely, at each time step $t$, the weight of expert $i$ is scaled by a factor $\exp\{-\eta\ell_{t,i}\}$ ($\ell_{t,i}$ is the loss incurred by the $i^{th}$ expert at time $t$) and then the weights are re-normalized to form a probability distribution. Intuitively, this exponentially down-weights experts with higher loss (or error), and slightly up-weights those with lower loss, concentrating more probability on better-performing predictors over time. Exponential Weights inspired many algorithms in the Online Learning literature, particularly in the adversarial bandits problem \cite{auer1995gambling,bubeck2012towards}, in the prediction with expert advice problem \cite{auer2002nonstochastic,mcmahan2009tighter} and in adversarial partial monitoring \cite{lattimore2019cleaning,digennaro2026instancedependentregretboundsnonstochastic}. In the context of logistic loss, EW is summarized in \cref{alg:EWA LogReg} below.

\begin{algorithm}[h]
\caption{EW for Logistic Loss}\label{alg:EWA LogReg}
\KwIn{Learning rate $\eta$, prior $\pi_0(\theta)$.}
\For{$t=1,\dots,n$}{
  The learner receives a new observation $x_t\in\mathcal{X}$.
  
  The learner predicts according to a distribution
  $$
    \hat{p}_t(x_t,y)
    \;=\;
    \mathbb{E}_{\theta\sim \rho_t}\bigl[\sigma\bigl(y\,\langle x_t,\theta\rangle\bigr)\bigr]=
    \int_{\mathbb{R}^d} \rho_t(\theta)\,\sigma\bigl(y\,\langle x_t,\theta\rangle\bigr)\,d\theta,
    $$
    
    with $y\in\{-1,+1\}$ and 
    $$\rho_t(\theta)= \frac{\exp \left\{\sum_{i=1}^{t-1}-\eta\ell_\theta(x_i,y_i)\right\}\pi_0(\theta) }{\int_{\mathbb{R}^d}\exp \left\{\sum_{i=1}^{t-1}-\eta \ell_{\theta'}(x_i,y_i)\right\}\pi_0(\theta')\,\mathrm{d}\theta'}\quad \mathrm{(posterior)}.$$ 
    
    Nature reveals the true label $y_t\in\{-1,1\}$ and the learner suffers the logistic loss $-\log(\hat{p}_t(x_t,y_t)).$
    
    Finally, the learner updates the distribution $\rho_t(\theta)$ into $\rho_{t+1}(\theta).$
}
\end{algorithm}

\section{Proofs of \cref{sec:ewa_regret_and_cheap_implementation}}
\label{app-sec: proofs of Theorems 1-2}

\subsection{Proof of \cref{th: regret guarantee of traditional EWA}}
\label{app-subsec: regret bounds theorems 1-2}

First of all, we introduce the following two well-known lemmas in the context of logistic regression and Exponential Weights algorithm.
\begin{lemma}[\citet{vovk2001competitive}, Lemma 1]
\label{lemma: Vovk's lemma}
    Exponential Weights with a logistic loss $\ell_\theta$ and a prior distribution $\pi_0(\theta)$ satisfies the following identity:
    \begin{equation}
        \sum_{t=1}^{T}-\frac{1}{\eta}\log\left(\mathbb{E}_{\theta\sim\rho_t}\exp(-\eta \ell_{\theta}(x_t,y_t))\right) 
        = -\frac{1}{\eta}\log\left(\mathbb{E}_{\theta\sim\pi_0}\exp\left(-\sum_{t=1}^{T}\eta \ell_{\theta}(x_t,y_t)\right)\right).
    \end{equation}
\end{lemma}

\begin{lemma}[\citet{donsker1975asymptotic}]
\label{lemma: Donsker-Varadhan variational formula}
    In the setup of Lemma \ref{lemma: Vovk's lemma}, let $\varphi$ be a distribution over $\mathbb{R}^d$. Then, the following identity holds:
    \begin{equation}
        -\frac{1}{\eta}\log\left(\mathbb{E}_{\theta\sim\pi_0}\exp\left(-\sum_{t=1}^{T}\eta \ell_{\theta}(x_t,y_t)\right)\right)
        = \inf_{\varphi}\left(\mathbb{E}_{\theta\sim\varphi}\sum_{t=1}^{T}\ell_{\theta}(x_t,y_t) + \frac{\mathcal{KL}(\varphi\|\pi_0)}{\eta}\right).
        \end{equation}
\end{lemma}

\EWArb*

\begin{proof}
Let us define
\[
\theta^*\in\underset{\substack{\theta \in \mathbb{R}^d\\ \|\theta\|\leq B}}{\arg\min}\ \sum_{t=1}^n -\log\left(\sigma(y_t\langle x_t,\theta\rangle)\right).
\]
Then, by definition of regret $R_n$ and by $1$-mixability of the logistic loss, we can write
\begin{align*}
    R_n &= -\sum_{t=1}^n \log\left(\hat{p}_t(x_t,y_t)\right) + \sum_{t=1}^n \log\left(\sigma(y_t\langle x_t,\theta^*\rangle)\right) \\
    &= \sum_{t=1}^n-\log\left(\mathbb{E}_{\theta\sim\rho_t}\left[\exp\left\{\log(\sigma(y_t\langle x_t,\theta\rangle))\right\}\right]\right) - \sum_{t=1}^n -\log\left(\sigma(y_t\langle x_t,\theta^*\rangle)\right).
\end{align*}
Then, by Lemma \ref{lemma: Vovk's lemma} and Lemma \ref{lemma: Donsker-Varadhan variational formula},
\begin{align*}
    R_n 
    &= \sum_{t=1}^n -\log\left(\mathbb{E}_{\theta\sim\rho_t}\left[\exp\left\{ \log(\sigma(y_t\langle x_t,\theta\rangle))\right\}\right]\right) - \sum_{t=1}^n -\log\left(\sigma(y_t\langle x_t,\theta^*\rangle)\right)\\
    &= -\log\left(\mathbb{E}_{\theta\sim\pi_0}\left[\exp\left\{\sum_{t=1}^n \log(\sigma(y_t\langle x_t,\theta\rangle))\right\}\right]\right) - \sum_{t=1}^n -\log\left(\sigma(y_t\langle x_t,\theta^*\rangle)\right)\\
    &= \inf_{\varphi}\left(\mathbb{E}_{\theta\sim\varphi}\left[-\sum_{t=1}^n \log(\sigma(y_t\langle x_t,\theta\rangle))\right]+ \mathcal{KL}(\varphi\|\pi_0)\right) - \sum_{t=1}^n -\log\left(\sigma(y_t\langle x_t,\theta^*\rangle)\right).
\end{align*}
The infimum is upper bounded by the particular choice of $\varphi\sim\mathcal{N}(\theta^*,\mu^2 I_d)$, with $\mu$ to be determined later. Further, let us set the prior as an isotropic Gaussian $\pi_0=\mathcal{N}(0,\lambda^2I_d)$, with $\lambda$ to be determined too. By computing the $\mathcal{KL}$ divergence between two Gaussian distributions, we can upper bound the regret $R_n$ as 
\begin{align*}
    R_n &\leq \inf_{\varphi}\left(\mathbb{E}_{\theta\sim\varphi}\left[-\sum_{t=1}^n \log(\sigma(y_t\langle x_t,\theta\rangle))\right] + \mathcal{KL}(\varphi\|\pi_0)\right) - \sum_{t=1}^n -\log\left(\sigma(y_t\langle x_t,\theta^*\rangle)\right)\\
    &\leq \mathbb{E}_{\theta\sim\mathcal{N}(\theta^*,\mu^2I_d)}\left[-\sum_{t=1}^n \log(\sigma(y_t\langle x_t,\theta\rangle))\right] + \mathcal{KL}\left(\mathcal{N}(\theta^*,\mu^2I_d)\|\mathcal{N}(0,\lambda^2I_d)\right) - \sum_{t=1}^n -\log\left(\sigma(y_t\langle x_t,\theta^*\rangle)\right).
\end{align*}
Using the fact that the function $u\mapsto -\log(\sigma(u))$ has second derivative bounded by $1/4$, we obtain
\begin{align*}
    \mathbb{E}_{\theta\sim\mathcal{N}(\theta^*,\mu^2I_d)}\left[-\sum_{t=1}^n \log(\sigma(y_t\langle x_t,\theta\rangle))\right]
    \le
    \sum_{t=1}^n -\log\left(\sigma(y_t\langle x_t,\theta^*\rangle)\right)
    + \mu^2 \sum_{t=1}^n \frac{\|x_t\|^2}{8}.
\end{align*}
Therefore,
\begin{align*}
    R_n
    &\leq \mu^2 \sum_{t=1}^n \frac{\|x_t\|^2}{8}
    + \mathcal{KL}\left(\mathcal{N}(\theta^*,\mu^2I_d)\|\mathcal{N}(0,\lambda^2I_d)\right) \\
    &= \mu^2 \sum_{t=1}^n \frac{\|x_t\|^2}{8} + \frac{1}{2}\left(d\log\left(\frac{\lambda^2}{\mu^2}\right)-d+\frac{d\mu^2}{\lambda^2}+\frac{\|\theta^*\|^2}{\lambda^2}\right).
\end{align*}
Under the assumptions that $\|x_t\|\leq R$ for all $t=1,\dots,n$ and $\|\theta^*\|\leq B$,
\begin{align*}
    R_n &\leq \mu^2 \sum_{t=1}^n \frac{\|x_t\|^2}{8} + \frac{1}{2}\left(d\log\left(\frac{\lambda^2}{\mu^2}\right)-d+\frac{d\mu^2}{\lambda^2}+\frac{\|\theta^*\|^2}{\lambda^2}\right)\\
    &\leq \mu^2 n\  \frac{R^2}{8} + \frac{1}{2}\left(d\log\left(\frac{\lambda^2}{\mu^2}\right)-d+\frac{d\mu^2}{\lambda^2}+\frac{B^2}{\lambda^2}\right) \\
    &= \mu^2 n\ \frac{R^2}{8} + \frac{1}{2}\left(d\log\left(\frac{B^2}{\mu^2}\right)-d+\frac{d\mu^2}{B^2}+1\right),
\end{align*}
where the last equality follows from the choice $\lambda=B$. Finally, if we optimize the last expression with respect to $\mu$ we get
\[
R_n \leq \frac{d}{2}\ \log\left(1+\frac{R^2B^2n}{4d}\right) + \frac{1}{2}.
\]
\end{proof}

\subsection{\cref{th: regret guarantee of COMPUTABLE EWA}}

\subsubsection{Regret of the computable EW implementation}

We first prove (in separate lemmas) that each of the steps in the EW implementation --- TV-shift, Monte Carlo approximation, smoothing --- only introduce additive terms to the regret.

\begin{lemma}
\label{lemma: regret additive constant smoothing}
    Suppose that a strategy $\hat{p}_t(x_t,y_t)=\mathbb{E}_{\theta\sim\rho_t}[\sigma(y_t \langle x_t, \theta\rangle)]$ for $t=1,\dots,n$ satisfies a regret inequality 
    \begin{equation*}
        \sum_{t=1}^n -\log\left(\hat{p}_t(x_t,y_t)\right) -\inf_{\|\theta\|\leq B} \ \sum_{t=1}^n -\log\left(\sigma(y_t\langle x_t,\theta\rangle)\right) \leq R\left(\{\hat{p}_t\}_{t=1}^n\right).
    \end{equation*}
    Then, for $0\leq\alpha_t\leq\frac{1}{2}$, the strategy $\hat{p}_t^{\alpha_t}(x_t,y_t)=(1-\alpha_t)\hat{p}_t(x_t,y_t)+\frac{\alpha_t}{2}$, $t=1,\dots,n$ satisfies
    \begin{equation}
    \label{eq: equation regret additive constant smoothing}
        \sum_{t=1}^n -\log\left(\hat{p}^{\alpha_t}_t(x_t,y_t)\right) -\inf_{\|\theta\|\leq B} \ \sum_{t=1}^n -\log\left(\sigma(y_t\langle x_t,\theta\rangle)\right) \leq R\left(\{\hat{p}_t\}_{t=1}^n\right)+2 \sum_{t=1}^n\alpha_t.
    \end{equation}
\end{lemma}

\begin{proof}
For $\alpha_t \in [0,\tfrac{1}{2}]$ it holds that
    \begin{align*}
        R_n &= \sum_{t=1}^n -\log\left(\hat{p}_t^{\alpha_t}(x_t,y_t)\right) -\inf_{\substack{\theta \in \mathbb{R}^d \\ \|\theta\|\leq B}} \ \sum_{t=1}^n -\log\left(\sigma(y_t\langle x_t,\theta\rangle)\right) \\
        &= \sum_{t=1}^n -\log\left(\hat{p}_t(x_t,y_t)\right) -\inf_{\substack{\theta \in \mathbb{R}^d \\ \|\theta\|\leq B}} \ \sum_{t=1}^n -\log\left(\sigma(y_t\langle x_t,\theta\rangle)\right) \\
        &\qquad + \sum_{t=1}^n -\log\left(\hat{p}_t^{\alpha_t}(x_t,y_t)\right) - \sum_{t=1}^n -\log\left(\hat{p}_t(x_t,y_t)\right) \\
        &\leq R\left(\{\hat{p}_t\}_{t=1}^n\right) + \sum_{t=1}^n -\log\left(\frac{\hat{p}_t(x_t,y_t)}{(1-\alpha_t)\hat{p}_t(x_t,y_t)+\frac{\alpha_t}{2}}\right)\\
        &\leq R\left(\{\hat{p}_t\}_{t=1}^n\right) + \sum_{t=1}^n \log\left(\frac{1}{1-\alpha_t}\right) \\
        &\leq R\left(\{\hat{p}_t\}_{t=1}^n\right) +2 \sum_{t=1}^n \alpha_t.
    \end{align*}
\end{proof}

\begin{lemma}
\label{lemma: regret additive constant TV-shift}
    Suppose that a strategy
    \[
    \hat{p}_t^{\alpha_t}(x_t,y_t)=(1-\alpha_t)\mathbb{E}_{\theta\sim\rho_t}[\sigma(y_t\langle x_t,\theta\rangle)] + \frac{\alpha_t}{2}, \quad t=1,\dots,n
    \]
    satisfies a regret inequality 
    \begin{equation*}
        \sum_{t=1}^n -\log\left(\hat{p}_t^{\alpha_t}(x_t,y_t)\right) -\inf_{\substack{\theta \in \mathbb{R}^d \\ \|\theta\|\leq B}} \ \sum_{t=1}^n -\log\left(\sigma(y_t\langle x_t,\theta\rangle)\right) \leq R\left(\{\hat{p}^{\alpha_t}_t\}_{t=1}^n\right).
    \end{equation*}
    Then, given an $E_t$-TV approximation $\tilde{\rho}_t$ of $\rho_t$, i.e.
    \[
    d_{\mathrm{TV}}(\rho_t,\tilde{\rho}_t)\leq E_t \qquad \text{for all } t=1,\dots,n,
    \]
    the strategy
    \[
    \tilde{p}_t^{\alpha_t}= (1-\alpha_t)\mathbb{E}_{\theta\sim\tilde{\rho}_t}[\sigma(y_t\langle x_t,\theta \rangle)] + \frac{\alpha_t}{2}, \quad t=1,\dots,n
    \]
    satisfies 
    \[
        \sum_{t=1}^n -\log\left(\tilde{p}_t^{\alpha_t}(x_t,y_t)\right) -\inf_{\substack{\theta \in \mathbb{R}^d \\ \|\theta\|\leq B}} \ \sum_{t=1}^n -\log\left(\sigma(y_t\langle x_t,\theta\rangle)\right) \leq R\left(\{\hat{p}^{\alpha_t}_t\}_{t=1}^n\right) + \sum_{t=1}^n\frac{2 E_t}{\alpha_t}.
    \]
\end{lemma}

\begin{proof}
Since $d_{\mathrm{TV}}(\rho_t,\tilde{\rho}_t)\leq E_t$ and the function
\[
\theta \mapsto \sigma(y_t\langle x_t,\theta\rangle)
\]
takes values in $[0,1]$, we have
\[
\left|\mathbb{E}_{\theta\sim\rho_t}[\sigma(y_t\langle x_t,\theta\rangle)]
-\mathbb{E}_{\theta\sim\tilde{\rho}_t}[\sigma(y_t\langle x_t,\theta\rangle)]\right|
\le d_{\mathrm{TV}}(\rho_t,\tilde{\rho}_t)
\le E_t.
\]
Hence,
\[
|\hat{p}_t^{\alpha_t}(x_t,y_t)-\tilde{p}_t^{\alpha_t}(x_t,y_t)|\le E_t.
\]

Thus,
\begin{align*}
    R_n &= \sum_{t=1}^n -\log\left(\tilde{p}_t^{\alpha_t}(x_t,y_t)\right) -\inf_{\substack{\theta \in \mathbb{R}^d \\ \|\theta\|\leq B}} \ \sum_{t=1}^n -\log\left(\sigma(y_t\langle x_t,\theta\rangle)\right) \\
    &= \sum_{t=1}^n -\log\left(\hat{p}_t^{\alpha_t}(x_t,y_t)\right) -\inf_{\substack{\theta \in \mathbb{R}^d \\ \|\theta\|\leq B}} \ \sum_{t=1}^n -\log\left(\sigma(y_t\langle x_t,\theta\rangle)\right) \\
    &\qquad + \sum_{t=1}^n -\log\left(\tilde{p}_t^{\alpha_t}(x_t,y_t)\right) - \sum_{t=1}^n -\log\left(\hat{p}_t^{\alpha_t}(x_t,y_t)\right) \\
    &\leq R(\{\hat{p}_t^{\alpha_t}\}_{t=1}^n) + \sum_{t=1}^n \log\left(\frac{\hat{p}_t^{\alpha_t}(x_t,y_t)}{\tilde{p}_t^{\alpha_t}(x_t,y_t)}\right).
\end{align*}
Moreover, since
\[
|\hat{p}_t^{\alpha_t}(x_t,y_t)-\tilde{p}_t^{\alpha_t}(x_t,y_t)|\le E_t
\qquad\text{and}\qquad
\tilde{p}_t^{\alpha_t}(x_t,y_t)\ge \frac{\alpha_t}{2},
\]
we have
\[
\frac{\hat{p}_t^{\alpha_t}(x_t,y_t)}{\tilde{p}_t^{\alpha_t}(x_t,y_t)}
\le
1+\frac{E_t}{\tilde{p}_t^{\alpha_t}(x_t,y_t)}
\le
1+\frac{2E_t}{\alpha_t}.
\]
Therefore,
\[
\log\left(\frac{\hat{p}_t^{\alpha_t}(x_t,y_t)}{\tilde{p}_t^{\alpha_t}(x_t,y_t)}\right)
\le
\log\left(1+\frac{2E_t}{\alpha_t}\right)
\le
\frac{2E_t}{\alpha_t},
\]
and thus
\[
R_n \le R(\{\hat{p}_t^{\alpha_t}\}_{t=1}^n) + \sum_{t=1}^n \frac{2E_t}{\alpha_t}.
\]
\end{proof}

\begin{lemma}
\label{lemma: regret additive constant multiplicative chernoff}
    Suppose that a strategy
    \[
    \tilde{p}_t^{\alpha_t}= (1-\alpha_t)\mathbb{E}_{\theta\sim\tilde{\rho}_t}[\sigma(y_t\langle x_t,\theta \rangle)] + \frac{\alpha_t}{2}, \quad t=1,\dots,n
    \]
    satisfies a regret inequality 
    \begin{equation*}
        \sum_{t=1}^n -\log\left(\tilde{p}_t^{\alpha_t}(x_t,y_t)\right) -\inf_{\substack{\theta \in \mathbb{R}^d \\ \|\theta\|\leq B}} \ \sum_{t=1}^n -\log\left(\sigma(y_t\langle x_t,\theta\rangle)\right) \leq R\left(\{\tilde{p}^{\alpha_t}_t\}_{t=1}^n\right).
    \end{equation*}
    Then, the strategy
    \[
    \tilde{p}_t^{\alpha_t,s_t}= (1-\alpha_t)\frac{1}{s_t} \sum_{i=1}^{s_t} \sigma(y_t\langle x_t,\theta_{t,i} \rangle) + \frac{\alpha_t}{2}, \quad t=1,\dots,n
    \]
    with $\theta_{t,i} \iid \tilde{\rho}_t$ for $i=1,\dots,s_t$, satisfies with probability at least $1-\sum_{t=1}^n\delta_t$
    \[
        \sum_{t=1}^n -\log\left(\tilde{p}_t^{\alpha_t,s_t}(x_t,y_t)\right) -\inf_{\substack{\theta \in \mathbb{R}^d \\ \|\theta\|\leq B}} \ \sum_{t=1}^n -\log\left(\sigma(y_t\langle x_t,\theta\rangle)\right) \leq R\left(\{\tilde{p}^{\alpha_t}_t\}_{t=1}^n\right) + 4\sum_{t=1}^n \sqrt{\frac{\log(1/\delta_t)}{\alpha_t s_t}}.
    \]
\end{lemma}

\begin{proof}
Let us define the following variables
\[
Z_{t,i}:=(1-\alpha_t)\,\sigma\big(y_t\langle \theta_{t,i},x_t\rangle\big)+\tfrac{\alpha_t}{2}\in\Big[\tfrac{\alpha_t}{2},\,1-\tfrac{\alpha_t}{2}\Big],
\quad
\tilde p_t^{\alpha_t, s_t}=\frac{1}{s_t}\sum_{i=1}^{s_t}Z_{t,i},\quad
\mu_t:=\mathbb{E}[Z_{t,1}]=\tilde p_t^{\alpha_t}\ge \alpha_t/2.
\]
Then, for any $\gamma_t\in(0,1]$,
\[
\mathbb{P} \left(-\log\frac{\tilde p_t^{\alpha_t,s_t}}{\tilde p_t^{\alpha_t}}>\gamma_t\right)
= \mathbb{P} \left(\tilde p_t^{\alpha_t,s_t}<e^{-\gamma_t}\mu_t\right)
\le \mathbb{P} \left(\tilde p_t^{\alpha_t,s_t}<(1-\tfrac{\gamma_t}{2})\mu_t\right)
\le \exp\Big(-\tfrac{\gamma_t^2}{8}\,s_t\,\mu_t\Big)
\le \exp\Big(-\tfrac{\gamma_t^2}{16}\,s_t\,\alpha_t\Big),
\]
using $e^{-\gamma_t}\le 1-\gamma_t/2$ and the standard multiplicative Chernoff bound for $[0,1]$ averages.
Choose
\[
\gamma_t = \ 4\sqrt{\frac{\log(1/\delta_t)}{\alpha_t s_t}},
\]
so that, assuming $\gamma_t\leq 1,$
\[
\mathbb{P} \left(-\log\frac{\tilde p_t^{\alpha_t,s_t}}{\tilde p_t^{\alpha_t}}>\gamma_t\right)\le \delta_t.
\]
A union bound over $t=1,\dots,n$ yields (with probability at least $1-\sum_{t=1}^n\delta_t$)
\[
\sum_{t=1}^n \big(-\log \tilde p_t^{\alpha_t,s_t} + \log \tilde p_t^{\alpha_t}\big)
\le \sum_{t=1}^n \gamma_t
\leq 4\sum_{t=1}^n \sqrt{\frac{\log(1/\delta_t)}{\alpha_t s_t}}.
\]
Thus, 
\begin{align*}
    R_n &= \sum_{t=1}^n -\log(\tilde{p}_t^{\alpha_t, s_t}(x_t,y_t)) -\inf_{\substack{\theta \in \mathbb{R}^d \\ \|\theta\|\leq B}} \ \sum_{t=1}^n -\log\left(\sigma(y_t\langle x_t,\theta\rangle)\right) \\
    &\leq R\left(\{\tilde{p}^{\alpha_t}_t\}_{t=1}^n\right) + \sum_{t=1}^n \Big(-\log \tilde{p}_t^{\alpha_t,s_t}(x_t,y_t) + \log \tilde{p}_t^{\alpha_t}(x_t,y_t) \Big) \\
    &\leq R\left(\{\tilde{p}^{\alpha_t}_t\}_{t=1}^n\right) + 4\sum_{t=1}^n \sqrt{\frac{\log(1/\delta_t)}{\alpha_t s_t}}. 
\end{align*}
\end{proof}

\subsubsection{The cost of sampling}
\label{app-subsec: computational complexity}

In this section, we discuss in details how (and what is the worst-case complexity) of approximating the posterior $\rho_t$ with a distribution $\tilde{\rho}_t$ such that
\[
d_{\mathrm{TV}}(\rho_t,\tilde{\rho}_t)\leq \sum_{i=1}^t \varepsilon_{i},
\]
with $\varepsilon_i$ being the \emph{fresh} budget accuracy of the sampler at time step $i$. First of all, Consider the posterior
\[
\rho_t(d\theta)\ \propto\ \exp\big\{-V_t(\theta)\big\}\,d\theta,
\qquad
V_t(\theta)\ :=\ \sum_{i=1}^{t-1} \underbrace{-\log\sigma\big(y_i\langle x_i,\theta\rangle\big)}_{\text{logistic loss}}
+\ \frac{1}{2B^2}\|\theta\|_2^2,
\]
whose Hessian is
\[
\nabla^2 V_t(\theta)\ =\ \sum_{i=1}^{t-1}\sigma\big(y_i\langle x_i,\theta\rangle\big)\big(1-\sigma\big(y_i\langle x_i,\theta\rangle\big)\big)\,x_i x_i^\top\ +\ \frac{1}{B^2}I_d.
\]
Assuming $\|x_i\|\le R$ for all $i\in[n]$, the potential $V_t$ is $L$--smooth and $m$--strongly convex with
\[
m=\frac{1}{B^2},
\qquad
L \ \le\ \frac{R^2}{4}(t-1)+\frac{1}{B^2},
\qquad
\kappa_t:=\frac{L}{m}\ =\ 1+\frac{R^2B^2}{4}(t-1).
\]
Then, in order to sample from such a log-concave posterior distribution $\rho_t$, we use the \emph{Metropolis Adjusted Langevin Algorithm} (MALA).

\paragraph{Warm-start MALA.}
For rounds $t\ge 2$, the target at time $t$ is a one-point multiplicative tilt of $\rho_{t-1}$:
\[
\rho_t(d\theta)\ \propto\ g_{t-1}(\theta)\,\rho_{t-1}(d\theta),
\qquad
g_{t-1}(\theta)=\sigma\big(y_{t-1}\langle x_{t-1},\theta\rangle\big)\in(0,1].
\]
The base case $t=1$ is exact since $\rho_1=\pi_0$. Rather than jumping directly from $\rho_{t-1}$ to $\rho_t$, we use a \emph{power-tempered ladder}
\[
\rho_{t,v}(d\theta)\ \propto\ g_{t-1}(\theta)^{\,v}\,\rho_{t-1}(d\theta),
\qquad v\in[0,1],
\]
discretized on the grid $0=v_0<v_1<\dots<v_K=1$ with uniform step $\Delta:=v_j-v_{j-1}$.
We choose $\Delta=\frac{c_\Delta}{RB}$ (small absolute $c_\Delta\in(0,1)$), so
$K=\lceil1/\Delta\rceil=\Theta(RB)$ and consecutive rungs are \emph{Rényi-warm}. In particular, for probability measures $P\ll Q$,
\[
D_2(P\|Q)\ :=\ \log \int \Big(\frac{dP}{dQ}\Big)^{2} dQ
\ =\ \log\,\mathbb{E}_Q\left[\Big(\frac{dP}{dQ}\Big)^{2}\right].
\]
If $P\ll M$ and $Q\ll M$ for a common dominating measure $M$, then
$\frac{dP}{dQ}=\frac{dP/dM}{dQ/dM}$, and the identity above still holds. In addition, we will repeatedly use the following \emph{change-of-measure} formula: for any integrable $f$,
\[
\mathbb{E}_{\rho_{t,v+\Delta}}[f]
=\frac{\int f(\theta) e^{(v+\Delta)Y(\theta)}\,\rho_{t-1}(d\theta)}{\int e^{(v+\Delta)Y(\theta)}\,\rho_{t-1}(d\theta)}
=\frac{Z_v}{Z_{v+\Delta}}\,\mathbb{E}_{\rho_{t,v}}\big[f(\theta)\,e^{\Delta Y(\theta)}\big],
\]
where $Y(\theta):=\log g_{t-1}(\theta)$ and $Z_v:=\int e^{vY}\,d\rho_{t-1}$ is the normalizer.

\begin{lemma}
\label{lem:adjacent-renyi-reverse}
Let $t\ge 2$, let $Y(\theta)=\log\sigma\big(y_{t-1}\langle x_{t-1},\theta\rangle\big)$ and
$\rho_{t,v}(d\theta)\propto e^{vY(\theta)}\,\rho_{t-1}(d\theta)$, $v\in[0,1]$.
Assume $\|x_{t-1}\|\le R$ and $\rho_{t-1}$ is $m$--strongly log--concave with $m=1/B^2$.
Then for any $\Delta>0$ with $v+\Delta\le1$,
\[
D_2\big(\rho_{t,v}\,\|\,\rho_{t,v+\Delta}\big)\ \le\ \Delta^2\,R^2B^2.
\]
\end{lemma}

\begin{proof}
Write $\rho_{t,v}(d\theta)=Z_v^{-1}e^{vY(\theta)}\,\rho_{t-1}(d\theta)$.
Then
\[
\frac{d\rho_{t,v}}{d\rho_{t,v+\Delta}}(\theta)
=\frac{Z_{v+\Delta}}{Z_v}\,e^{-\Delta Y(\theta)}.
\]
Therefore
\begin{align*}
D_2(\rho_{t,v}\|\rho_{t,v+\Delta})
&=\log\,\E_{\rho_{t,v+\Delta}}\left[\Big(\tfrac{d\rho_{t,v}}{d\rho_{t,v+\Delta}}\Big)^2\right]
=\log\left(\frac{Z_{v+\Delta}^2}{Z_v^2}\,\E_{\rho_{t,v+\Delta}}\big[e^{-2\Delta Y}\big]\right).
\end{align*}
Apply the change-of-measure identity with $f=e^{-2\Delta Y}$:
\[
\E_{\rho_{t,v+\Delta}}\big[e^{-2\Delta Y}\big]
=\frac{Z_v}{Z_{v+\Delta}}\,\E_{\rho_{t,v}}\big[e^{-\Delta Y}\big].
\]
Hence
\[
D_2(\rho_{t,v}\|\rho_{t,v+\Delta})
=\log\Big( \underbrace{\tfrac{Z_{v+\Delta}}{Z_v}}_{=\,M_v(\Delta)}\cdot
\underbrace{\E_{\rho_{t,v}}\big[e^{-\Delta Y}\big]}_{=\,M_v(-\Delta)} \Big)
=\log\big(M_v(\Delta)\,M_v(-\Delta)\big),
\]
where $M_v(\lambda):=\E_{\rho_{t,v}}[e^{\lambda Y}]$ is the mgf of $Y$ under $\rho_{t,v}$.\\
Now, observe that $Y$ is $R$--Lipschitz; indeed, for $z:=y_{t-1}\langle x_{t-1},\theta\rangle$,
\[
\nabla_\theta Y(\theta)=(1-\sigma(z))\,y_{t-1}\,x_{t-1}
\quad\Rightarrow\quad \|\nabla Y(\theta)\|\le\|x_{t-1}\|\le R.
\]
Each rung $\rho_{t,v}$ is $m$--strongly log--concave: its potential is $V_{t-1}-vY$,
and $Y$ is concave (since $(\log\sigma)''(z)=-\sigma(z)(1-\sigma(z))\le0$), so $-vY$ is convex and
$\nabla^2(V_{t-1}-vY)\succeq mI$.
Thus, the centered variable $\tilde Y:=Y-\E_{\rho_{t,v}}Y$ is sub-Gaussian with proxy
\[
\log \E_{\rho_{t,v}} e^{\lambda \tilde Y}\ \le\ \frac{\lambda^2}{2}\cdot\frac{R^2}{m}
\ =\ \frac{\lambda^2}{2}\,R^2B^2,\qquad \forall \lambda\in\mathbb{R}.
\]
Finally, decompose $M_v(\pm\Delta)=e^{\pm\Delta\,\E_{\rho_{t,v}}Y}\,\E_{\rho_{t,v}}e^{\pm\Delta\tilde Y}$ and apply the previous bound to get
\[
\log M_v(\pm\Delta)\ \le\ \pm\Delta\,\E_{\rho_{t,v}}Y\ +\ \frac{\Delta^2}{2}\,R^2B^2.
\]
Summing the $+\Delta$ and $-\Delta$ inequalities cancels the linear terms and gives
\[
\log\big(M_v(\Delta)\,M_v(-\Delta)\big)\ \le\ \Delta^2 R^2B^2,
\]
which is exactly $D_2(\rho_{t,v}\|\rho_{t,v+\Delta})$ by the identity above.
\end{proof}

Choosing $\Delta=c_\Delta/(RB)$ makes the right-hand side a fixed constant, i.e., each rung is a
\emph{constant Rényi warm start}. Warm-start MALA on $m$--strongly log--concave targets (step size
$h\asymp 1/(L\sqrt d)$) then mixes to $\varepsilon_t$-TV accuracy in
$\tilde{\mathcal O}(\sqrt d\,\kappa_t\log(1/\varepsilon_t))$ oracle calls.

\paragraph{Accuracy budgeting across rungs (with approximate warm start).}
Let $Q_j$ be the Markov kernel of $N_{\mathrm{rung}}$ MALA steps at rung $j$ (target $\rho_{t,v_j}$).
Instead of initializing from the exact $\rho_{t,v_0}=\rho_{t-1}$, we start from its approximate counterpart from the previous round, i.e.
\[
\tilde\nu_0=\tilde{\rho}_{t-1}.
\]
Define the iterates
\[
\tilde\nu_j \coloneqq \tilde\nu_{j-1} Q_j,\qquad j=1,\dots,K,
\]
and let the initial mismatch be
\[
\mathrm{err}_{t-1} := d_{\mathrm{TV}}(\tilde\nu_0,\rho_{t,v_0})
=
d_{\mathrm{TV}}(\tilde\rho_{t-1},\rho_{t-1}).
\]
If for each rung we ensure
\begin{equation}
\label{eq:per-rung-tv}
d_{\mathrm{TV}}(\rho_{t,v_{j-1}}Q_j,\rho_{t,v_j})\le \varepsilon_j,
\end{equation}
then by TV contraction under Markov kernels and the triangle inequality,
\[
d_{\mathrm{TV}}(\tilde\nu_j,\rho_{t,v_j})
\le
d_{\mathrm{TV}}(\tilde\nu_{j-1}Q_j,\rho_{t,v_{j-1}}Q_j)
+
d_{\mathrm{TV}}(\rho_{t,v_{j-1}}Q_j,\rho_{t,v_j})
\le
d_{\mathrm{TV}}(\tilde\nu_{j-1},\rho_{t,v_{j-1}}) + \varepsilon_j.
\]
Iterating over $j=1,\dots,K$ yields
\[
d_{\mathrm{TV}}(\tilde\nu_K,\rho_{t,v_K})
\le
\mathrm{err}_{t-1} + \sum_{j=1}^K \varepsilon_j.
\]
Since $\rho_{t,v_K}=\rho_t$, this gives
\[
\mathrm{err}_{t}\coloneqq d_{\mathrm{TV}}(\tilde\rho_t,\rho_t)\le \mathrm{err}_{t-1}+\sum_{j=1}^K \varepsilon_j.
\]
In particular, if we allocate a fresh round-$t$ budget $\varepsilon_t$ across the rungs so that $\sum_{j=1}^K \varepsilon_j \le \varepsilon_t$, then $\mathrm{err}_{t} \le\mathrm{err}_{t-1}+\varepsilon_t$.

\begin{proposition}
\label{prop:per-sample-bridged}
With $\Delta=c_\Delta/(RB)$ and $K=\lceil 1/\Delta\rceil$, let $\tilde\nu_0=\tilde{\rho}_{t-1}$ be the approximate distribution carried over from the previous round, and set $\tilde\nu_j := \tilde\nu_{j-1}Q_j,\qquad j=1,\dots,K,$ where $Q_j$ is the Markov kernel of $N_{\mathrm{rung}}$ MALA steps at rung $j$ (target $\rho_{t,v_j}$).
Denote the initialization mismatch by $\mathrm{err}_{t-1}:=d_{\mathrm{TV}}(\tilde\nu_0,\rho_{t,v_0})
= d_{\mathrm{TV}}(\tilde\rho_{t-1},\rho_{t-1})$. If we ensure the per-rung guarantee \eqref{eq:per-rung-tv} with $\varepsilon_j=\varepsilon_t/K$, then
\[
d_{\mathrm{TV}}(\tilde\nu_K,\rho_t) \le \mathrm{err}_{t-1} + \varepsilon_t.
\]
In particular, one approximate draw from a distribution within TV distance $\mathrm{err}_{t-1}+\varepsilon_t$ of $\rho_t$ costs
\[
Q_t\ =\ K\cdot N_{\mathrm{rung}}
\ =\ \tilde{O}\Big(\sqrt d\ \kappa_t\ K\ \log\tfrac{K}{\varepsilon_t}\Big),
\qquad
\kappa_t=1+\frac{R^2B^2}{4}(t-1).
\]
\end{proposition}

\begin{proof}
By Lemma~\ref{lem:adjacent-renyi-reverse}, $D_2(\rho_{t,v_{j-1}}\|\rho_{t,v_j})\le O(1)$ for each rung, so the warm-start term in the MALA complexity (when initialized from $\rho_{t,v_{j-1}}$) is a constant. Thus, ensuring TV error $\varepsilon_t/K$ at rung $j$ takes
\[
N_{\mathrm{rung}}=\tilde{\mathcal O}\Big(\sqrt d\,\kappa_t\log\big(\tfrac{K}{\varepsilon_t}\big)\Big)
\]
oracle calls. Summing over $K$ rungs gives the stated $Q_t$.

For the total error, apply the approximate warm-start budgeting argument: by TV contraction under Markov kernels and the triangle inequality,
\[
d_{\mathrm{TV}}(\tilde\nu_j,\rho_{t,v_j})
\le
d_{\mathrm{TV}}(\tilde\nu_{j-1},\rho_{t,v_{j-1}}) + \varepsilon_j,
\]
and iterating yields
\[
d_{\mathrm{TV}}(\tilde\nu_K,\rho_{t,v_K})
\le
\mathrm{err}_{t-1} + \sum_{j=1}^K \varepsilon_j
=
\mathrm{err}_{t-1}+\varepsilon_t.
\]
Since $\rho_{t,v_K}=\rho_t$, this proves
\[
d_{\mathrm{TV}}(\tilde\nu_K,\rho_t)\le \mathrm{err}_{t-1}+\varepsilon_t.
\]
\end{proof}

Consequently, if we denote \(\mathrm{err}_0=0\), Proposition~\ref{prop:per-sample-bridged}
yields the recursion
\[
\mathrm{err}_t \le \mathrm{err}_{t-1}+\varepsilon_t.
\]
Iterating over \(t=1,\dots,n\) gives
\[
\mathrm{err}_t \le \sum_{i=1}^t \varepsilon_i.
\]

Proposition~\ref{prop:per-sample-bridged} is a per-chain, per-output-sample guarantee: one independent execution of the bridged warm-start MALA construction produces one draw whose law is within TV distance $\mathrm{err}_{t-1}+\varepsilon_t$ of $\rho_t$, at cost $Q_t$. Hence, by running $s_t$ independent copies of this construction independently at round $t$, each initialized from its own previous-round state and using independent randomness, we obtain $s_t$ i.i.d. samples from the same approximate law $\tilde\rho_t$, with total round-$t$ cost $\tilde O(s_tQ_t)$.

\subsubsection{Proof of Theorem \ref{th: regret guarantee of COMPUTABLE EWA}}
\label{app-subsec: proof of theorem 2}

For completeness, we state here the multiplicative Chernoff bound used in the proof of \cref{th: regret guarantee of COMPUTABLE EWA}.

\begin{theorem}[Multiplicative Chernoff Bound]
\label{prop: multiplicative chernoff bound}
Let $X_1,\dots,X_s$ be independent random variables taking values in $[0,1]$, and let
\[
X = \sum_{i=1}^s X_i,\qquad \mu = \mathbb{E}[X].
\]
Then, for any $0<\gamma<1$,
\[
\mathbb{P}\bigl(X \le (1-\gamma)\,\mu\bigr)\le \exp\Bigl(-\,\frac{\gamma^2}{2}\,\mu\Bigr).
\]
\end{theorem}

\CEWAHighProbrb*

\begin{proof}
Recall that EW (Algorithm \ref{alg:EWA LogReg}) achieves a regret bound of order $O(d\log(Bn))$ by predicting according to $\hat{p}_t$ (see Theorem \ref{th: regret guarantee of traditional EWA}). Furthermore, we can decompose the regret as follows:
\begin{align*}
R_n &=\sum_{t=1}^n -\log \tilde p_t^{\alpha_t,s_t}(y_t)
- \inf_{\|\theta\|\le B}\sum_{t=1}^n -\log \sigma(y_t\langle\theta,x_t\rangle) \\
&= \underbrace{\sum_{t=1}^n -\log \hat p_t(y_t)
- \inf_{\|\theta\|\le B}\sum_{t=1}^n -\log \sigma(y_t\langle\theta,x_t\rangle)}_{\le\ \cO(d\log(Bn))} \\
&\quad + \underbrace{\sum_{t=1}^n \big(-\log \hat p_t^{\alpha_t}(y_t) + \log \hat p_t(y_t)\big)}_{\text{(A) smoothing}}
+ \underbrace{\sum_{t=1}^n \big(-\log \tilde p_t^{\alpha_t}(y_t) + \log \hat p_t^{\alpha_t}(y_t)\big)}_{\text{(B) TV shift}} \\
&\quad + \underbrace{\sum_{t=1}^n \big(-\log \tilde p_t^{\alpha_t,s_t}(y_t) + \log \tilde p_t^{\alpha_t}(y_t)\big)}_{\text{(C) sampling}} .
\end{align*}

Now, in order to bound (A), (B), (C), we use \cref{lemma: regret additive constant smoothing}, \cref{lemma: regret additive constant TV-shift,prop:per-sample-bridged}, and \cref{lemma: regret additive constant multiplicative chernoff}, respectively. Since $\sum_{t=1}^n \delta_t\le \delta$, the union bound implies that, with probability at least $1-\delta$, we have

\begin{equation*}
    R_n \leq O(d\log(Bn)) + \sum_{t=1}^n 2\alpha_t +\sum_{t=1}^n \frac{2\sum_{i=1}^t \varepsilon_i}{\alpha_t} + 4\sum_{t=1}^n  \sqrt{\frac{\log(\nicefrac{1}{\delta_t})}{s_t\alpha_t}}.
\end{equation*}

Given that we need $s_t$ i.i.d. samples at every time step, the total computational complexity of the algorithm is
\[
N_{\mathrm{total}} =\cO\left(\sum_{t=1}^n s_t Q_t\right),
\]
where $Q_t$ is the sampling complexity (measured as number of first-order oracle queries).
By Proposition~\ref{prop:per-sample-bridged}, the per-sample cost is
\[
Q_t \;=\; \tilde{\cO}\Big(\sqrt d\,\kappa_t\,K\,\log\tfrac{K}{\varepsilon_t}\Big),
\]
where $
K=\Theta(RB), \text{ and }
\kappa_t \;=\; 1+\frac{R^2B^2}{4}(t-1).$
\end{proof}

\CorollaryRBandCOST*

\begin{proof}
First, by the choice $\delta_t=\frac{\delta}{n}$, $\sum_{t=1}^n \delta_t =\delta$. We now bound the three approximation terms one by one.

\begin{itemize}
    \item \underline{Smoothing:} since $\alpha_t=\frac{1}{2n}$ we have that $2\alpha_t=\frac{1}{n}$.
    \item \underline{TV-shift:} by the choice of $\varepsilon_t$, $\sum_{i=1}^t \varepsilon_i = \frac{t}{20n^3}$. Therefore, summing over $t=1,\dots,n$, we obtain
    \[
    \sum_{t=1}^n \frac{2 \sum_{i=1}^t \varepsilon_i}{\alpha_t} =
    \frac{1}{5n^2}\sum_{t=1}^n t
    \leq \frac{1}{5}.
    \]
    In particular, the TV-shift contribution is $O(1)$. 
    \item \underline{Sampling:} by the choice of $s_t$,
    \[
    \alpha_t s_t
    \ge
    16\log\left(\frac{1}{\delta_t}\right).
    \]
    Hence,
    \[
    4\sqrt{\frac{\log(1/\delta_t)}{\alpha_t s_t}}
    \le \frac{1}{n}.
    \]
    Hence, \cref{th: regret guarantee of COMPUTABLE EWA} applies with probability at least $1-\delta$.
\end{itemize}  

Plugging these bounds (and summing over $n$) into \cref{eq: regret computable EWA}, we obtain
\[
R_n \le O(d\log(Bn)) + O(1).
\]
Regarding the complexity, by \cref{th: regret guarantee of COMPUTABLE EWA}, the per-sample cost is
\[
Q_t
=
\tilde{O}\left(\sqrt d\,\kappa_t\,K\,\log\frac{K}{\varepsilon_t}\right),
\qquad
\kappa_t=1+\frac{R^2B^2}{4}(t-1),
\]
where $K$ is the number of rungs. The fact that $\varepsilon_t=\frac{1}{20n^3}$ and $K=\Theta(RB)$ yields
\[
N_{\mathrm{total}}
=
\tilde{O}\left(\sum_{t=1}^n s_t Q_t\right)
=
\tilde{O}\big(\sqrt d\,R^3B^3\,n^5\big).
\]
\end{proof}

\subsection{Online-to-batch conversion}
\label{sec-app: online-to-batch conversion}

Online-to-batch conversion provides a principled way to transform guarantees from sequential prediction into statistical guarantees for batch learning. Early results established that online learnability can imply PAC learnability: in particular, \citet{littlestone1994weighted} showed that algorithms with finite mistake bounds can be converted into batch learners with provable generalization guarantees. Related ideas also appeared in algorithmic developments such as the voted perceptron of \citet{freund1998large}, where averaging or voting over online iterates yields a batch classifier with strong empirical performance and theoretical guarantees. More broadly, classical online methods such as Weighted Majority and Hedge \citep{littlestone1994weighted,freund1997decision} provided some of the conceptual foundations for later batch ensemble procedures. A general framework for online-to-batch conversion was formalized by \citet{cesa2004generalization}, who showed that a sequence of hypotheses generated by an online learner with small regret can be converted into a predictor with low expected risk in the stochastic setting. This result yields a generic reduction from online regret bounds to batch generalization guarantees. Subsequent work refined this principle in several directions. For example, \citet{zhang2004solving} derived sharper, data-dependent concentration guarantees for online-to-batch conversion. More recent contributions have focused on obtaining high-probability guarantees for the last iterate, rather than averaged predictors, and on adapting the conversion to structured optimization settings. In particular, \citet{pmlr-v97-cutkosky19a} developed an anytime reduction that transforms an arbitrary online learner into a stochastic optimization procedure whose last iterate achieves optimal rates without requiring prior knowledge of smoothness or noise parameters. It is therefore straightforward to adapt these results to logistic regression.

Finally, while EW has shown to be robust in the regime where the MLE does not exist (linearly separable data), we acknowledge here that an online-to-batch conversion brings into the excess risk bound a $\log(n)$ that \citet{mourtada2022improper} prove to be suboptimal. Indeed, the authors propose \emph{Sample Minimax Predictor (SMP)} that is still well defined when MLE is not and aggregates the results of only two Logistic Regressions, achieving a bound of order $\cO\left(\frac{d}{n}\right)$. However, this bound is polynomial with respect to the other parameters such as $B$ and $R$, leaving open if one can achieve optimal batch bound also for the other parameters of the problem. 

\section{Exponential Weights for $B\to\infty$}

\limitingEWA*

\begin{proof}
Write $\hat p_t^B$ as a ratio:
\[
\hat p_t^B(x,y)
=
\frac{\int_{\R^d}\sigma(y\langle x,\theta\rangle)\prod_{i<t}\sigma(y_i\langle x_i,\theta\rangle)
\,e^{-\|\theta\|^2/(2B^2)}\,d\theta}
{\int_{\R^d}\prod_{i<t}\sigma(y_i\langle x_i,\theta\rangle)
\,e^{-\|\theta\|^2/(2B^2)}\,d\theta}.
\]
Change variables $\theta=B\tilde\theta$ (so $d\theta=B^d\,d\tilde\theta$). The Jacobian cancels
between numerator and denominator, and $e^{-\|B\tilde\theta\|^2/(2B^2)}=e^{-\|\tilde\theta\|^2/2}$, hence
\[
\hat p_t^B(x,y)
=
\frac{\int_{\R^d}\sigma(B\,y\langle x,\tilde\theta\rangle)\prod_{i<t}\sigma(B\,y_i\langle x_i,\tilde\theta\rangle)
\,e^{-\|\tilde\theta\|^2/2}\,d\tilde\theta}
{\int_{\R^d}\prod_{i<t}\sigma(B\,y_i\langle x_i,\tilde\theta\rangle)
\,e^{-\|\tilde\theta\|^2/2}\,d\tilde\theta}.
\]
For any $a\in\R$,
\[
\sigma(Ba)\to \mathbf 1\{a>0\}+\frac12\,\mathbf 1\{a=0\}
\qquad\text{as }B\to\infty.
\]
Thus, for Lebesgue-a.e.\ $\tilde\theta$,
\[
\prod_{i<t}\sigma(B\,y_i\langle x_i,\tilde\theta\rangle)\to \mathbf 1\{\tilde\theta\in H_t\},
\]
and
\[
\sigma(B\,y\langle x,\tilde\theta\rangle)\to \mathbf 1\{y\langle x,\tilde\theta\rangle>0\}
+\frac12\,\mathbf 1\{y\langle x,\tilde\theta\rangle=0\},
\]
since the exceptional set for the first convergence is contained in a finite union of hyperplanes, hence has measure zero.
Moreover, both integrands are dominated by $e^{-\|\tilde\theta\|^2/2}$ because $\sigma(\cdot)\in[0,1]$.
By dominated convergence, the numerator and denominator converge to the corresponding integrals with
$\mathbf 1\{\tilde\theta\in H_t\}$ inserted.

Let $Z_\infty:=\int_{\R^d} e^{-\|u\|^2/2}\mathbf 1\{u\in H_t\}\,du$.
Then $0<Z_\infty<(2\pi)^{d/2}$ because $H_t$ is nonempty and open under strict separability.
Taking the limit in the ratio yields
\begin{align*}
\lim_{B\to\infty}\hat p_t^B(x,y)
&=
\frac{1}{Z_\infty}\int_{\R^d}e^{-\|\theta\|^2/2}\mathbf 1\{\theta\in H_t\}
\left(\mathbf 1\{y\langle x,\theta\rangle>0\}+\frac12\,\mathbf 1\{y\langle x,\theta\rangle=0\}\right)\,d\theta \\
&=
\mathbb{P}_{\theta\sim\rho_t^\infty}(y\langle x,\theta\rangle>0)
+\frac12\,\mathbb{P}_{\theta\sim\rho_t^\infty}(y\langle x,\theta\rangle=0),
\end{align*}
as claimed.
\end{proof}

Let us recall the definition of the (open) version cone
\[
H_t \coloneqq \big\{\theta\in\mathbb{R}^d:\ y_i\langle x_i,\theta\rangle>0\ \ \forall i<t\big\}.
\]
Then, for any $\gamma>0$, we denote the (closed) margin slice as
\[
S_{t,\gamma} \coloneqq \big\{\theta\in\mathbb{R}^d:\ y_i\langle x_i,\theta\rangle\ge \gamma\ \ \forall i<t\big\}.
\]
Note that $H_t=\bigcup_{\gamma>0} S_{t,\gamma}$ and each $S_{t,\gamma}$ is a nonempty closed convex
polyhedron whenever $H_t\neq\emptyset$.

\begin{lemma}
\label{lem:Scalings slice margin}
Assume strict separability up to time $t-1$, i.e., $H_t\neq\emptyset$. Then:
\begin{enumerate}
    \item $S_{t,1}\neq\emptyset$.
    \item For every $\gamma>0$,
    \( S_{t,\gamma} \;=\; \gamma\, S_{t,1} \coloneqq \big\{\gamma w:\ w\in S_{t,1}\big\}.
    \)
\end{enumerate}
\end{lemma}

\begin{proof}
(1) Since $H_t\neq\emptyset$, pick $u\in H_t$. Then all margins are strictly positive:
$y_i\langle x_i,u\rangle>0$ for $i<t$. Let
\[
\gamma_{\min} \coloneqq \min_{i<t}\, y_i\langle x_i,u\rangle>0,
\qquad
v \coloneqq \frac{u}{\gamma_{\min}}.
\]
For every $i<t$,
$y_i\langle x_i,v\rangle
= \frac{1}{\gamma_{\min}}\,y_i\langle x_i,u\rangle \ge 1$,
so $v\in S_{t,1}$ and $S_{t,1}\neq\emptyset$.

(2) We prove both inclusions.

\begin{itemize}
\item($\subseteq$) Let $\theta\in S_{t,\gamma}$ and set $w=\theta/\gamma$.
Then for all $i<t$,
\(
y_i\langle x_i,w\rangle
= \frac{1}{\gamma}\,y_i\langle x_i,\theta\rangle \ge 1,
\)
hence $w\in S_{t,1}$ and $\theta=\gamma w\in \gamma S_{t,1}$.

\item{($\supseteq$)} Let $w\in S_{t,1}$ and set $\theta=\gamma w$.
Then for all $i<t$,
\(
y_i\langle x_i,\theta\rangle
= \gamma\, y_i\langle x_i,w\rangle \ge \gamma,
\)
so $\theta\in S_{t,\gamma}$.
\end{itemize}

Combining the two inclusions gives $S_{t,\gamma}=\gamma S_{t,1}$.
\end{proof}

\noindent
In particular, the map $\psi_\gamma:S_{t,1}\to S_{t,\gamma}$, $\psi_\gamma(w)=\gamma w$,
is a bijection. Therefore any optimization over $S_{t,\gamma}$ can be equivalently
reparameterized over $S_{t,1}$ via $\theta=\gamma w$.

\EWAsvm*

\begin{proof}
By strict separability up to time $t-1$, the version cone $H_t$ is nonempty. For any margin $\gamma>0$, define the (closed) margin slice
\[
S_{t,\gamma}\coloneqq \{\theta\in\mathbb{R}^d:\ y_i\langle x_i,\theta\rangle\ge \gamma,\ \forall i<t\}.
\]
By \cref{lem:Scalings slice margin}, $S_{t,1}$ is nonempty and $S_{t,\gamma}=\gamma\,S_{t,1}$ for every
$\gamma>0$.

Now consider the truncated Gaussian
\[
\rho_{t,\gamma}^\infty(d\theta)\ \propto\ e^{-\|\theta\|^2/2}\,\mathds{1}\{\theta\in S_{t,\gamma}\}\,d\theta.
\]
Its (unnormalized) log--density is
\[
\log f_{t,\gamma}(\theta)
=
\begin{cases}
-\tfrac12\|\theta\|^2+\mathrm{const}, & \theta\in S_{t,\gamma},\\[2pt]
-\infty, & \text{otherwise}.
\end{cases}
\]
Maximizing $\log f_{t,\gamma}$ over $S_{t,\gamma}$ is thus equivalent to minimizing
$\tfrac12\|\theta\|^2$ over the nonempty closed convex set $S_{t,\gamma}$.
Since the objective is strictly convex, there exists a unique minimizer, which we denote by $\theta_{t,\gamma}$.
Equivalently, $\theta_{t,\gamma}$ is the Euclidean projection of $0$ onto $S_{t,\gamma}$.

Moreover, since $S_{t,\gamma}=\gamma S_{t,1}$, the map $\psi_\gamma:S_{t,1}\to S_{t,\gamma}$ defined by
$\psi_\gamma(w)=\gamma w$ is a bijection. Reparameterizing via $\theta=\gamma w$ yields
\[
\arg\min_{\theta\in S_{t,\gamma}} \tfrac12\|\theta\|^2
\;=\;
\arg\min_{w\in S_{t,1}} \tfrac12\|\gamma w\|^2
\;=\;
\arg\min_{w\in S_{t,1}} \tfrac12\|w\|^2.
\]
Hence the minimizer in $w$--space is
\[
w_{t,\mathrm{svm}} \coloneqq \arg\min_{w\in\mathbb{R}^d}\ \tfrac12\|w\|^2
\quad\text{s.t.}\quad y_i\langle x_i,w\rangle\ge 1,\ \forall i<t,
\]
i.e., the (unique) hard--margin SVM solution at time $t-1$. Mapping back via $\theta=\psi_\gamma(w)$ gives
\[
\theta_{t,\gamma}\;=\;\gamma\,w_{t,\mathrm{svm}}.
\]
In particular, the direction of the mode,
$\theta_{t,\gamma}/\|\theta_{t,\gamma}\| = w_{t,\mathrm{svm}}/\|w_{t,\mathrm{svm}}\|$, is independent of $\gamma$
and coincides with the hard--margin SVM direction.
\end{proof}

\ModeLimitSmallMargin*

\begin{proof}
First of all, if $0<\gamma'<\gamma$, then $S_{t,\gamma}\subseteq S_{t,\gamma'}$.
Moreover, we have
\begin{equation}\label{eq:union_slices_is_cone}
\bigcup_{\gamma>0} S_{t,\gamma}=H_t.
\end{equation}
Indeed, if $\theta\in H_t$ then $m(\theta):=\min_{i<t}y_i\langle x_i,\theta\rangle>0$, hence
$\theta\in S_{t,\gamma}$ for all $\gamma\in(0,m(\theta)]$, proving $H_t\subseteq\cup_{\gamma>0}S_{t,\gamma}$.
The reverse inclusion is immediate since $S_{t,\gamma}\subseteq H_t$ for every $\gamma>0$.

Now, define
\[
f_\gamma(\theta):=e^{-\|\theta\|^2/2}\mathbf 1\{\theta\in S_{t,\gamma}\},
\qquad
f_\infty(\theta):=e^{-\|\theta\|^2/2}\mathbf 1\{\theta\in H_t\}.
\]
By the nesting of the slices and \eqref{eq:union_slices_is_cone}, we have $f_\gamma(\theta)\uparrow f_\infty(\theta)$
pointwise as $\gamma\downarrow 0$. Since $f_\infty\in L^1(\R^d)$, the monotone convergence theorem yields
\[
Z_{t,\gamma}:=\int_{\R^d} f_\gamma(\theta)\,d\theta\ \uparrow\ \int_{\R^d} f_\infty(\theta)\,d\theta=:Z_t^\infty\in(0,\infty).
\]
Let $p_\gamma:=f_\gamma/Z_{t,\gamma}$ and $p_\infty:=f_\infty/Z_t^\infty$ be the corresponding normalized densities.
Then for every $\theta$,
\[
p_\gamma(\theta)\ \longrightarrow\ p_\infty(\theta)\qquad(\gamma\downarrow 0),
\]
and, since $0\le p_\gamma(\theta)\le f_\infty(\theta)/Z_{t,\gamma}$ and $Z_{t,\gamma}\ge Z_{t,\gamma_0}>0$
for all $\gamma\le\gamma_0$, we can apply dominated convergence to obtain
\[
\int_{\R^d}|p_\gamma(\theta)-p_\infty(\theta)|\,d\theta\ \longrightarrow\ 0.
\]
Recalling that $\|\rho_{t,\gamma}^\infty-\rho_t^\infty\|_{\mathrm{TV}}
=\tfrac12\int |p_\gamma-p_\infty|$, this proves the first part of the claim.

On $H_t$, the (unnormalized) log-density of $\rho_t^\infty$ is $-\tfrac12\|\theta\|^2+\mathrm{const}$.
Since $H_t$ is a cone, for any $\theta\in H_t$ and any $r\in(0,1)$ we have $r\theta\in H_t$, and thus
\[
\sup_{\theta\in H_t}\Big(-\tfrac12\|\theta\|^2\Big)=0,
\]
but this supremum is not attained because $\theta=0\notin H_t$ (strict inequalities). Hence $\rho_t^\infty$
has no mode. Finally, for each $\gamma>0$, \cref{prop:mode-svm} gives that $\rho_{t,\gamma}^\infty$ has a unique mode
$\theta_{t,\gamma}=\gamma w_{t,\mathrm{svm}}$, completing the second claim.
\end{proof}

\subsection{Linear Separability}
\label{app-sec: linear separability margin based bound}

\begin{lemma}
\label{lemma:cap-in-cone}
Assume the data are linearly separable, i.e. there exists a \emph{unit} vector $u\in\mathbb S^{d-1}$ and $\gamma>0$ such that $y_t\langle u,x_t\rangle\ge \gamma$ for all $t=1,\dots,n$. Let $R:=\max_t \|x_t\|$ and $\bar\gamma:=\gamma/R$. For any unit vector $v\in\mathbb S^{d-1}$ write $v=a u + b w$ with $a=\langle v,u\rangle\in[-1,1]$, $b=\sqrt{1-a^2}\ge 0$, and $w\in\mathbb S^{d-1}$ satisfying $\langle w,u\rangle=0$. Then, for all $t$,
\[
y_t\langle v,x_t\rangle \;\ge\; a\,\gamma \;-\; b\,R.
\]
In particular, if $t_0:=\frac{1}{\sqrt{1+\bar\gamma^2}}\in(0,1)$, the spherical cap
\[
\mathrm{Cap}(u,t_0):=\{v\in\mathbb S^{d-1}:\ \langle v,u\rangle\ge t_0\}
\]
is contained in the version cone intersected with the sphere,
$\,\mathrm{Cap}(u,t_0)\subseteq H\cap\mathbb S^{d-1}$, where
$H:=\{\theta\in\mathbb R^d:\ y_t\langle \theta,x_t\rangle\ge 0,\ \forall t\}$.
\end{lemma}

\begin{proof}
For any $v=a u + b w$ with $u,w$ orthonormal, we have
\[
y_t\langle v,x_t\rangle
= a\,y_t\langle u,x_t\rangle + b\,y_t\langle w,x_t\rangle
\ \ge\ a\,\gamma - b\,\|x_t\|
\ \ge\ a\,\gamma - b\,R,
\]
using the margin condition and Cauchy–Schwarz. If moreover $a\ge t_0$ then
$b\le \sqrt{1-t_0^2}=\bar\gamma/\sqrt{1+\bar\gamma^2}$, hence
\[
y_t\langle v,x_t\rangle \ \ge\ \gamma t_0 - R\sqrt{1-t_0^2}
= \frac{\gamma - R\bar\gamma}{\sqrt{1+\bar\gamma^2}} = 0,
\]
so $v\in H$. Therefore $\mathrm{Cap}(u,t_0)\subseteq H\cap\mathbb S^{d-1}$.
\end{proof}

\begin{lemma}
\label{lemma: Proba Cap lower bound}
Let $d\ge 2$, let $U\sim\mathrm{Unif}(\mathbb{S}^{d-1})$, and fix a unit vector $u\in\mathbb S^{d-1}$. Then, for any $t\in[0,1]$,
\[
\mathbb{P}\big(\langle U,u\rangle\ge t\big)
\;\ge\;
\frac{\Gamma\left(\frac{d}{2}\right)}{(d-1)\sqrt{\pi}\ \Gamma\left(\frac{d-1}{2}\right)}
\,(1-t^2)^{\frac{d-1}{2}}
\;\ge\; c_d\,(1-t^2)^{\frac{d-1}{2}},
\]
for a constant $c_d>0$ depending only on $d$.
\end{lemma}

\begin{proof}
By rotational invariance, the marginal $T:=\langle U,u\rangle$ has density
\[
f_d(s)\;=\;\frac{\Gamma(\frac d2)}{\sqrt{\pi}\,\Gamma(\frac{d-1}{2})}\,(1-s^2)^{\frac{d-3}{2}},
\qquad -1<s<1,
\]
since $\int_{-1}^{1}(1-s^2)^{\frac{d-3}{2}}\,ds=\frac{\sqrt{\pi}\,\Gamma(\frac{d-1}{2})}{\Gamma(\frac d2)}$.
Therefore
\[
\mathbb{P}\{\langle U,u\rangle\ge t\}
=\int_t^1 f_d(s)\,ds
=\frac{\Gamma(\frac d2)}{\sqrt{\pi}\,\Gamma(\frac{d-1}{2})}\int_t^1 (1-s^2)^{\frac{d-3}{2}}\,ds.
\]
Set $r=1-s^2$ so $ds=-\frac{dr}{2s}$ and $s=\sqrt{1-r}$. Then
\[
\int_t^1 (1-s^2)^{\frac{d-3}{2}}\,ds
=\frac12\int_{0}^{\,1-t^2}\frac{r^{\frac{d-3}{2}}}{s}\,dr
\ \ge\ \frac12\int_{0}^{\,1-t^2} r^{\frac{d-3}{2}}\,dr
=\frac{1}{d-1}\,(1-t^2)^{\frac{d-1}{2}},
\]
where we used that $s\in[t,1]\subset(0,1]$ so $1/s\ge 1$. Multiplying by the prefactor yields the first inequality, and the second follows by defining
$c_d:=\Gamma(\frac d2)/\big((d-1)\sqrt{\pi}\,\Gamma(\frac{d-1}{2})\big)$.
\end{proof}

\begin{lemma}
\label{lemma: lower bound likelihood}
    For $z\geq0, \sigma(z)=\frac{1}{1+e^{-z}}\geq 1-e^{-z}$. Hence, if $\min_{t=1,...,n} y_t \langle\theta, x_t\rangle \geq m\geq 0$
    \[
    \prod_{t=1}^n \sigma(y_t\langle \theta, x_t \rangle)\geq (1-e^{-m})^n.
    \]
\end{lemma}

\begin{proof}
For $a\ge 0$, $(1+a)^{-1}\ge 1-a$. With $a=e^{-z}$ and $z\ge 0$ this gives $\sigma(z)\ge 1-e^{-z}$. If each term satisfies $y_t\langle \theta,x_t\rangle\ge m$, then $\sigma(y_t\langle \theta,x_t\rangle)\ge 1-e^{-m}$ for all $t$, and the product bound follows.
\end{proof}

\begin{lemma}
\label{lemma: lower bound margin in C t_1 lambda}
Assume separability: there exists a \emph{unit} vector $u\in\mathbb S^{d-1}$ with
$y_i\langle u,x_i\rangle\ge \gamma>0$ for all $i=1,\dots,n$, and let
$R:=\max_i \|x_i\|$ and $\bar\gamma:=\gamma/R$. For any $t_1\in(t_0,1)$, where
$t_0:=1/\sqrt{1+\bar\gamma^2}$, define
\[
\alpha(t_1):=\gamma t_1 - R\sqrt{1-t_1^2}\;>\;0,
\qquad
\mathrm{Cap}(u,t_1):=\{v\in\mathbb S^{d-1}:\ \langle v,u\rangle\ge t_1\},
\]
and for $\lambda_0>0$ the cone–ray set
\[
C_{t_1,\lambda_0}:=\{\theta=\lambda v:\ v\in \mathrm{Cap}(u,t_1),\ \lambda\ge \lambda_0\}.
\]
Then, for every $\theta=\lambda v\in C_{t_1,\lambda_0}$,
\[
\min_{1\le i\le n} y_i\langle \theta,x_i\rangle \;\ge\; \lambda\,\alpha(t_1).
\]
In particular, if $\lambda_0 \ge m/\alpha(t_1)$ for some $m>0$, then for all $\theta\in C_{t_1,\lambda_0}$,
\[
\prod_{i=1}^n \sigma\big(y_i \langle\theta, x_i\rangle\big)
\;\ge\; \sigma(m)^n \;\ge\; \big(1-e^{-m}\big)^n.
\]
\end{lemma}

\begin{proof}
Fix $v\in\mathrm{Cap}(u,t_1)\subset\mathbb S^{d-1}$ and write $v=a\,u+b\,w$ with
$a=\langle v,u\rangle\ge t_1$, $b=\sqrt{1-a^2}\le \sqrt{1-t_1^2}$, and $w\in\mathbb S^{d-1}$, $\langle w,u\rangle=0$.
For each $i$,
\[
y_i\langle v,x_i\rangle
= a\,y_i\langle u,x_i\rangle + b\,y_i\langle w,x_i\rangle
\ \ge\ a\,\gamma - b\,\|x_i\|
\ \ge\ \gamma t_1 - R\sqrt{1-t_1^2}
= \alpha(t_1).
\]
By homogeneity, for $\theta=\lambda v$ we get
$y_i\langle \theta,x_i\rangle=\lambda\,y_i\langle v,x_i\rangle\ge \lambda\,\alpha(t_1)$, hence the margin bound.
If $\lambda\ge \lambda_0\ge m/\alpha(t_1)$ then $\min_i y_i\langle \theta,x_i\rangle\ge m$, so
$\sigma\big(y_i\langle \theta,x_i\rangle\big)\ge \sigma(m)\ge 1-e^{-m}$ for all $i$, and the product bound follows.
\end{proof}

\begin{lemma}
\label{lemma: Prob Z in C t_1 lambda}
Let $Z\sim \mathcal{N}(0,B^2I_d)$ and define $U:=Z/\|Z\|\in\mathbb S^{d-1}$ and $S:=\|Z\|/B\in(0,\infty)$. Then $U$ and $S$ are independent with
\[
U\sim \mathrm{Unif}(\mathbb S^{d-1}),\qquad
S\sim \chi_d\ \text{ with density }\ 
p_S(s)=\frac{1}{2^{\frac d2-1}\Gamma(\frac d2)}\,s^{d-1}e^{-s^2/2},\ s>0.
\]
Moreover, for any $t_1\in(0,1)$ and $\lambda_0>0$,
\[
\mathbb{P}\big(Z\in C_{t_1,\lambda_0}\big)
\;=\;
\mathbb{P}\big(U\in \mathrm{Cap}(u,t_1)\big)\cdot \mathbb{P}\left(S\ge \frac{\lambda_0}{B}\right),
\]
where $C_{t_1,\lambda_0}=\{\theta=\lambda v:\ v\in \mathrm{Cap}(u,t_1),\ \lambda\ge \lambda_0\}$.
\end{lemma}

\begin{proof}
Let $N\sim\mathcal N(0,I_d)$. Then $Z=BN$ and $N$ admits the polar decomposition $N=SU$ with
$U\sim\mathrm{Unif}(\mathbb S^{d-1})$, $S=\|N\|\sim \chi_d$, and $U\perp S$ (spherical symmetry).
Thus $Z=(BS)\,U$ with $U\perp S$ and $S=\|Z\|/B$. By definition,
\[
\{Z\in C_{t_1,\lambda_0}\}
\ =\ \{\exists\, \lambda\ge \lambda_0,\ v\in \mathrm{Cap}(u,t_1):\ Z=\lambda v\}
\ =\ \{\,U\in \mathrm{Cap}(u,t_1),\ \|Z\|\ge \lambda_0\,\}
\]
because $U=Z/\|Z\|$ is the direction and $\|Z\|$ the radius. Since $\|Z\|\ge \lambda_0$ is the event $\{BS\ge \lambda_0\}=\{S\ge \lambda_0/B\}$ and $U\perp S$, we obtain
\[
\mathbb P(Z\in C_{t_1,\lambda_0})=\mathbb P\big(U\in \mathrm{Cap}(u,t_1)\big)\,\mathbb P\left(S\ge \frac{\lambda_0}{B}\right).
\]
The stated marginal laws of $U$ and $S$ are standard (chi/chi-square radius and uniform angle) and follow from the Gaussian polar change of variables; the density of $S$ is 
$p_S(s)=\frac{1}{2^{\frac d2-1}\Gamma(\frac d2)}\,s^{d-1}e^{-s^2/2}$ for $s>0$.
\end{proof}

\begin{lemma}
\label{lem:alpha-lb}
For all $\bar\gamma\in(0,1]$,
\(
\alpha(t_1)\ \ge\ \frac{\gamma}{\,2(2+\sqrt{2})\,}.
\)
\end{lemma}

\begin{proof}
Let $h(t):=\bar\gamma t-\sqrt{1-t^2}$ so that $\alpha(t)=R\,h(t)$. Since $\sqrt{1-t^2}$ is concave on $[0,1]$,
for $t\ge t_0$ we have
\[
\sqrt{1-t^2}\ \le\ \sqrt{1-t_0^2}
+\Big(\frac{d}{dt}\sqrt{1-t^2}\Big)\Big|_{t=t_0}\,(t-t_0)
\ =\ \bar\gamma t_0 - \frac{1}{\bar\gamma}(t-t_0).
\]
Thus $h(t)\ge (\bar\gamma+\bar\gamma^{-1})(t-t_0)$. Evaluating at $t_1=(1+t_0)/2$ gives
\[
h(t_1)\ \ge\ (\bar\gamma+\bar\gamma^{-1})\frac{1-t_0}{2}.
\]
Since
\(
1-t_0
=1-\frac{1}{\sqrt{1+\bar\gamma^2}}
=\frac{\bar\gamma^2}{\sqrt{1+\bar\gamma^2}\,\big(1+\sqrt{1+\bar\gamma^2}\big)}
\ge \frac{\bar\gamma^2}{2+\sqrt{2}},
\)
we obtain
\(
h(t_1)\ \ge\ \frac{\bar\gamma}{\,2(2+\sqrt{2})\,},
\)
and multiplying by $R$ yields the claim.
\end{proof}

\begin{theorem}[Full version of Theorem \ref{thm:ewa-interpolation-simple}]
\label{thm:ewa-interpolation-simple-threshold}
Assume that $d\geq2$ and that data are linearly separable with margin $\gamma>0$. Furthermore, let $\|x_t\|\le R$, and write $\bar\gamma:=\gamma/R\in(0,1]$. Now, fix the following quantities
\[
t_0\ :=\ \frac{1}{\sqrt{1+\bar\gamma^2}},
\qquad
t_1\ :=\ \frac{1+t_0}{2}\in(0,1),
\qquad
\alpha(t)\ :=\ \gamma t - R\sqrt{1-t^2}.
\]
Finally, define $\lambda_0:=\log(2n)/\alpha(t_1)$ and let $\pi_0=\mathcal N(0,B^2I_d)$.
Then, for every $B>0$, the EW cumulative predictive loss satisfies
\[
e^{-L_n^B}
\ \ge\
e^{-1}\,\mathbb{P}\big(U\in \mathrm{Cap}(u,t_1)\big)\cdot
\mathbb{P}\Big(S\ge \tfrac{\lambda_0}{B}\Big),
\]
where $Z\sim\pi_0$ has polar decomposition $Z=(BS)U$ with
$U\sim\mathrm{Unif}(\mathbb S^{d-1})$, $S\sim\chi_d$, $U\perp S$. \\
Equivalently, with $c_d>0$ depending only on $d$, and with the particular choice of $t_1$
\[
L_n^B
\ \le\
(d-1)\log\Big(\frac{2}{\bar{\gamma}}\Big)
\;+\; \mathsf{Rad}(B,\lambda_0)\;+\;O\big(\log(d)\big),
\qquad
\mathsf{Rad}(B,\lambda_0):=
\begin{cases}
\frac{1}{2}\big(\tfrac{\lambda_0}{B}\big)^2 + O\big(d\log(d)\big), & \text{if } B\le \tfrac{\lambda_0}{\sqrt{d-1}},
\\[4pt]
\log(2), & \text{if } B\ge \tfrac{\lambda_0}{\sqrt{d-1}}.
\end{cases}
\]
In particular, define the margin–based threshold
\[
B_{\mathrm{critic}}
\ :=\
\frac{\kappa\,\log(2n)}{\gamma\sqrt{d-1}},
\qquad
\kappa\ \ge\ 2(2+\sqrt{2}).
\]
Hence $B\ge B_{\mathrm{critic}}\ \Rightarrow\ B\ge \lambda_0/\sqrt{d-1}$ and
\[
L_n^B
\ \le\
(d-1)\,\log\Big(\frac{2}{\bar\gamma}\Big)\;+\;O(\log(d)).
\]
\end{theorem}

\begin{proof}
Let $Z\sim\mathcal N(0,B^2I_d)$ and write its polar decomposition $Z=(BS)U$ with
$U:=Z/\|Z\|\sim\mathrm{Unif}(\mathbb S^{d-1})$, $S:=\|Z\|/B\sim\chi_d$, and $U\perp S$
(Lemma~\ref{lemma: Prob Z in C t_1 lambda}).
By Lemma~\ref{lemma: lower bound margin in C t_1 lambda}, for any
$\theta=\lambda v\in C_{t_1,\lambda_0}$ we have
$\min_i y_i\langle \theta,x_i\rangle \ge \lambda\,\alpha(t_1) \ge \lambda_0\,\alpha(t_1) = \log(2n)$.
Hence, by Lemma~\ref{lemma: lower bound likelihood},
\[
\prod_{i=1}^n \sigma\big(y_i\langle \theta,x_i\rangle\big)
\ \ge\ \sigma(\log(2n))^n
\ \ge\ e^{-1}\,.
\]
Using the EW telescoping identity
$e^{-L_n^B}=\int \prod_{i=1}^n \sigma(y_i\langle \theta,x_i\rangle)\, \pi_0(d\theta)$
(with $\pi_0=\mathcal N(0,B^2 I_d)$) and restricting to $C_{t_1,\lambda_0}$,
\[
e^{-L_n^B}\ \ge\ e^{-1}\,\pi_0\big(C_{t_1,\lambda_0}\big)
\ =\ e^{-1}\,\mathbb{P}\big(U\in \mathrm{Cap}(u,t_1)\big)\;\mathbb{P}\Big(S\ge \frac{\lambda_0}{B}\Big),
\]
by the factorization in Lemma~\ref{lemma: Prob Z in C t_1 lambda}. Further, by
Lemma~\ref{lemma: Proba Cap lower bound},
\[
\mathbb{P}\big(U\in \mathrm{Cap}(u,t_1)\big)
\ \ge\ c_d\,\big(1-t_1^2\big)^{\frac{d-1}{2}}.
\]
For the radial term, set $r:=\lambda_0/B>0$ and split two cases.

\emph{Case (i): $B\le \lambda_0/\sqrt{d-1}$, i.e., $r\ge \sqrt{d-1}$.}
For $S\sim\chi_d$, a one–slice lower bound on the tail of the chi density
$p_S(s)=\frac{1}{2^{\frac d2-1}\Gamma(\frac d2)} s^{d-1}e^{-s^2/2}$ gives
\[
\mathbb{P}(S\ge r) \ \ge\ \int_r^{r+1/r} p_S(s)ds \ \ge\ \underbrace{\frac{e^{-3/2}}{2^{d/2-1}\Gamma(\tfrac{d}{2})}}_{c_d^\chi}\  r^{d-2} e^{-r^2/2}\,,
\]
for all $r\ge \sqrt{d-1}$. Consequently,
\begin{align*}
    -\log \mathbb{P} \Big(S\ge \frac{\lambda_0}{B}\Big)
    \ &\le\ \frac{1}{2}\Big(\frac{\lambda_0}{B}\Big)^2 - \log(c_d^{\chi})\\
    &=\ \frac{1}{2}\Big(\frac{\lambda_0}{B}\Big)^2 + \frac{3}{2} + \Big(\frac{d}{2}-1\Big)\log2 + \log \Gamma\left(\frac{d}{2}\right) \\
    &=\ \frac{1}{2}\Big(\frac{\lambda_0}{B}\Big)^2 + \Big(\frac{d}{2}-\frac{1}{2}\Big)\log d -\frac{d}{2} + \Big(\frac{3}{2}+\frac{1}{2}\log\pi\Big)+o(1),
\end{align*}
after dropping the favorable $-(d-2)\log r$ term and applying Stirling's formula.

\emph{Case (ii): $B\ge \lambda_0/\sqrt{d-1}$, i.e., $r\le \sqrt{d-1}$.}
By monotonicity of the tail,
\[
\mathbb P(S\ge r)\ \ge\ \mathbb P\big(S\ge \sqrt{d-1}\big)
= \mathbb P(\chi_d^2\ge d-1).
\]
Now let $M_d$ denote the median of $\chi_d^2$. Since $\chi_d^2\overset{d}=2\,\Gamma(d/2,1)$ and the median
$\nu(k)$ of $\Gamma(k,1)$ satisfies the classical bound $\nu(k)>k-\tfrac13$ \citep{bmkmedianGamma}, we obtain
\[
M_d = 2\,\nu(d/2) > 2\Big(\frac d2-\frac13\Big)=d-\frac23>d-1
\qquad (d\ge2).
\]
Because $\chi_d^2$ is continuous, $\mathbb P(\chi_d^2\ge M_d)\geq \frac12$. Therefore,
\[
\mathbb P(\chi_d^2\ge d-1)\ge \mathbb P(\chi_d^2\ge M_d)\geq\frac12,
\]
and hence
\[
-\log \mathbb P(S\ge r)\le \log 2.
\]

Combining the two above cases and taking $-\log$ we get,
\begin{align*}
L_n^B &\le -\log\Big(c_d(1-t_1^2)^{\frac{d-1}{2}}\Big)-\log \mathbb{P}\left(S\ge \frac{\lambda_0}{B}\right)+ 1 \\
&= \frac{d-1}{2}\log\left(\frac{1}{1-t_1^2}\right)
+ \mathsf{Rad}(B,\lambda_0)-\log(c_d) \\
&= \frac{d-1}{2}\log\left(\frac{1}{1-t_1^2}\right)
+ \mathsf{Rad}(B,\lambda_0) + \frac{1}{2}\ \log\big(2\pi(d-1)\big) + o(1),
\end{align*}
with $\mathsf{Rad}(B,\lambda_0)$ as stated.

Finally, with $t_1=\frac{1+t_0}{2}$ and $t_0=1/\sqrt{1+\bar\gamma^2}$ we have
$1-t_1^2 \ge \frac{1}{2}(1-t_0^2)=\frac{\bar\gamma^2}{2(1+\bar\gamma^2)}\ge \frac{\bar\gamma^2}{4}$,
hence the angular term is at most $(d-1)\log(\frac{2}{\bar\gamma})$.
Moreover, by Lemma \ref{lem:alpha-lb}, for such a $t_1$ we have
\[
\alpha(t_1)\ \ge\ \frac{\gamma}{2(2+\sqrt{2})},
\]
uniformly in $\bar\gamma\in(0,1]$. Therefore,
\[
B\ \ge\ B_{\mathrm{critic}}\ :=\ \frac{\kappa\,\log(2n)}{\gamma\sqrt{d-1}}
\quad\text{with}\quad \kappa\ge 2(2+\sqrt{2}),
\]
implies that $B\ge \lambda_0/\sqrt{d-1}$. Thus, we are in Case (ii), which yields
\[
L_n^B\ \le\ (d-1)\,\log\Big(\frac{2}{\bar\gamma}\Big)+O\big(\log(d)\big).
\]
\end{proof}

\begin{lemma}
\label{lemma:hp-R-delta}
Assume $x_t \stackrel{\mathrm{iid}}{\sim} \mathcal N(0,I_d)$ for $t=1,\dots,n$. 
Then, for any $\delta\in(0,1)$, with probability at least $1-\delta$,
\[
R:=\max_{t\le n}\|x_t\|
\ \le\ 
\sqrt d\;+\;\sqrt{2\log\frac{n}{\delta}}\,.
\]
\end{lemma}

\begin{proof}
For $X\sim\mathcal N(0,I_d)$ the map $x\mapsto\|x\|$ is $1$-Lipschitz, hence for any $u>0$,
\(
\mathbb P(\|X\|\ge \mathbb E\|X\|+u)\le e^{-u^2/2}.
\)
Since $\mathbb E\|X\|\le \sqrt d$, for each $t$ we have
\(
\mathbb P(\|x_t\|\ge \sqrt d + u)\le e^{-u^2/2}.
\)
By a union bound over $t=1,\dots,n$,
\(
\mathbb P(R\ge \sqrt d + u)\le n\,e^{-u^2/2}.
\)
Choose $u=\sqrt{2\log(n/\delta)}$ to get $\mathbb P(R\ge \sqrt d + u)\le \delta$.
\end{proof}

\begin{lemma}
\label{lemma:strong-signal-separability-delta}
Let $x_t \stackrel{\mathrm{iid}}{\sim}\mathcal N(0,I_d)$ for $t=1,\dots,n$ and suppose the labels follow the logistic model
\[
\mathbb P(y_t=1\mid x_t)=\sigma(\langle x_t,\theta^*\rangle),\qquad \sigma(t)=\frac{1}{1+e^{-t}},
\]
with $\|\theta^*\|=B$ and $y_t\in\{\pm1\}$ for all $t$. Define $u:=\theta^*/\|\theta^*\|$.
Fix $\delta_1,\delta_2\in(0,1)$ and assume
\[
B \;\ge\; \frac{2n}{\sqrt{2\pi}\,\delta_1}\,.
\]
Then, with probability at least $1-\delta_1-\delta_2$, the dataset is linearly separable by the hyperplane normal to $u$, and its (unnormalized) margin satisfies
\[
\gamma\ :=\ \min_{1\le t\le n}\, y_t\,\langle u, x_t\rangle
\ \ge\ \frac{\delta_2}{2n}\,.
\]
\end{lemma}

\begin{proof}
Write $z_t:=\langle x_t,\theta^*\rangle$ and $Z\sim\mathcal N(0,1)$.
Conditional on $x_t$, the flip event $\{y_t \neq \operatorname{sign}(z_t)\}$ has probability
\(
\mathbb P(y_t \neq \operatorname{sign}(z_t)\mid x_t)=\sigma(-|z_t|)\le e^{-|z_t|}.
\)
Unconditionally,
\[
p_{\mathrm{flip}}
=\mathbb P(y_t \neq \operatorname{sign}(z_t))
\le \mathbb E[e^{-|z_t|}]
= \mathbb E[e^{-B|Z|}]
= \frac{2}{\sqrt{2\pi}} \int_0^\infty e^{-B z} e^{-z^2/2}\,dz
= \frac{2}{\sqrt{2\pi}}\, e^{B^2/2}\int_B^\infty e^{-t^2/2}\,dt.
\]
Using $\int_a^\infty e^{-t^2/2}\,dt \le a^{-1}e^{-a^2/2}$ for $a>0$ gives
\(
p_{\mathrm{flip}}\le \tfrac{2}{\sqrt{2\pi}}\cdot \tfrac{1}{B}.
\)
A union bound over $t=1,\dots,n$ yields
\[
\mathbb P\Big(\exists\,t:\ y_t \neq \operatorname{sign}\langle x_t,\theta^*\rangle\Big)
\ \le\ n\,p_{\mathrm{flip}}
\ \le\ \frac{2n}{\sqrt{2\pi}\,B}
\ \le\ \delta_1,
\]
by the assumed lower bound on $B$. Denote this “no flip” event by $\mathcal E_1$.

On $\mathcal E_1$ we have $y_t=\operatorname{sign}\langle x_t,\theta^*\rangle$ for all $t$, hence the sample is separated by $u$ with margin
\(
\gamma=\min_{t} |\langle u,x_t\rangle|.
\)
Since $\langle u, x_t\rangle \sim \mathcal N(0,1)$, for any $\varepsilon\in(0,1)$,
\(
\mathbb P(|\langle u,x_t\rangle|\le \varepsilon) \le 2\varepsilon.
\)
By a union bound over $t$,
\[
\mathbb P\Big(\min_{1\le t\le n} |\langle u,x_t\rangle| \le \varepsilon\Big)
\ \le\ 2n\varepsilon.
\]
Choosing $\varepsilon=\tfrac{\delta_2}{2n}$ gives
\(
\mathbb P(\gamma \le \delta_2/(2n)) \le \delta_2.
\)
Let $\mathcal E_2$ be the complementary event, so that on $\mathcal E_2$ we have $\gamma\ge \tfrac{\delta_2}{2n}$.

Finally,
\(
\mathbb P(\mathcal E_1\cap \mathcal E_2)\ge 1-\delta_1-\delta_2,
\)
and on $\mathcal E_1\cap\mathcal E_2$ the claims hold.
\end{proof}

\EWArblargeBiidUnified*

\begin{proof}
Define the events
\[
\mathcal E_1:=\Big\{\forall t,\ y_t=\operatorname{sign}\langle x_t,\theta^*\rangle\Big\},\quad
\mathcal E_2:=\Big\{\min_{t\le n}|\langle u,x_t\rangle|\ge \tfrac{\delta_2}{2n}\Big\},\quad
\mathcal E_3:=\Big\{\max_{t\le n}\|x_t\|\le \sqrt d+\sqrt{2\log\tfrac{n}{\delta_3}}\Big\},
\]
with $u:=\theta^*/\|\theta^*\|$. By Lemma~\ref{lemma:strong-signal-separability-delta} and the first term in
\eqref{eq:unified-B}, $\mathbb P(\mathcal E_1)\ge 1-\delta_1$. On $\mathcal E_1$ the sample is separated by $u$ and
$\gamma=\min_t y_t\langle u,x_t\rangle=\min_t |\langle u,x_t\rangle|$. Again by
Lemma~\ref{lemma:strong-signal-separability-delta}, $\mathbb P(\mathcal E_2)\ge 1-\delta_2$. By
Lemma~\ref{lemma:hp-R-delta}, $\mathbb P(\mathcal E_3)\ge 1-\delta_3$. A union bound yields
\[
\mathbb P(\mathcal E_1\cap \mathcal E_2\cap \mathcal E_3)\ \ge\ 1-\delta_1-\delta_2-\delta_3.
\]

On $\mathcal E:=\mathcal E_1\cap \mathcal E_2\cap \mathcal E_3$ we have
$\gamma\ge \delta_2/(2n)$ and $R\le \sqrt d+\sqrt{2\log(n/\delta_3)}$. To activate the $B$-independent branch of
Theorem~\ref{thm:ewa-interpolation-simple-threshold} it suffices to verify $B\ge \lambda_0/\sqrt{\,d-1\,}$, where
$\lambda_0=\frac{\log(2n)}{\alpha(t_1)}$ and
$\alpha(t)=\gamma t - R\sqrt{1-t^2}$, $t_1=\tfrac{1+t_0}{2}$, $t_0=1/\sqrt{1+\bar\gamma^2}$, $\bar\gamma=\gamma/R$.
By Lemma~\ref{lem:alpha-lb}, $\alpha(t_1)\ge \gamma/[2(2+\sqrt{2})]$, hence on $\mathcal E$
\[
\frac{\lambda_0}{\sqrt{\,d-1\,}}
\;=\;\frac{\log(2n)}{\alpha(t_1)\sqrt{\,d-1\,}}
\;\le\; \frac{2(2+\sqrt{2})\,\log(2n)}{\gamma\,\sqrt{\,d-1\,}}
\;\le\; \frac{4(2+\sqrt{2})\,n\,\log(2n)}{\delta_2\,\sqrt{\,d-1\,}}.
\]
By the second term in \eqref{eq:unified-B}, $B$ exceeds this quantity, so the $B$-independent branch applies and yields
\[
L_n^B\ \le\ (d-1)\,\log\Big(\frac{2}{\bar\gamma}\Big)+O\big(\log(d)\big),
\qquad \bar\gamma=\gamma/R.
\]
Finally, using $\gamma\ge \delta_2/(2n)$ and $R\le \sqrt d+\sqrt{2\log(n/\delta_3)}$, a simple algebra gives \eqref{eq:ewa-delta-final-unified}. Since $R_n\le L_n^B$ in the separable case, the same upper bound holds for $R_n$. Fixing $\delta_1=\delta_2 = \delta_3 = \tfrac{\delta}{3}$ completes the proof.
\end{proof}

\section{Additional experiments}
\label{sec-app: Experiments}

We now validate the theoretical results on a simple synthetic dataset. Similarly to \citet{rudi_binary_efficient_logreg}, we evaluate the performance of EW on the \emph{adversarial} 1-D dataset presented by \citet{hazan2014logistic}. In particular, the data $\{(x_t,y_t)\}_{t=1}^n$ are generated in an i.i.d. fashion as follows
\begin{equation*}
    (x_t, y_t) = 
    \begin{cases}
        (1-\frac{\sqrt{\varepsilon}}{2B}, \ 1), \ \text{w.p. } \frac{\sqrt{\varepsilon}}{2B}+\chi\frac{\varepsilon}{B} \\
        (\frac{\sqrt{\varepsilon}}{2B}, \ -1), \ \text{otherwise}
    \end{cases}.
\end{equation*}
We set $n\in\{75, 100, 150, 200, 300, 400\}$, $B=\log(n)$, $\varepsilon=0.01$ and $\chi\in\{\pm 1\}$. The experiment is averaged over $70$ simulations for $\chi=1$ and another $70$ for $\chi=-1$. 

\begin{figure}[h] 
  \centering
  \includegraphics[width=0.6\columnwidth]{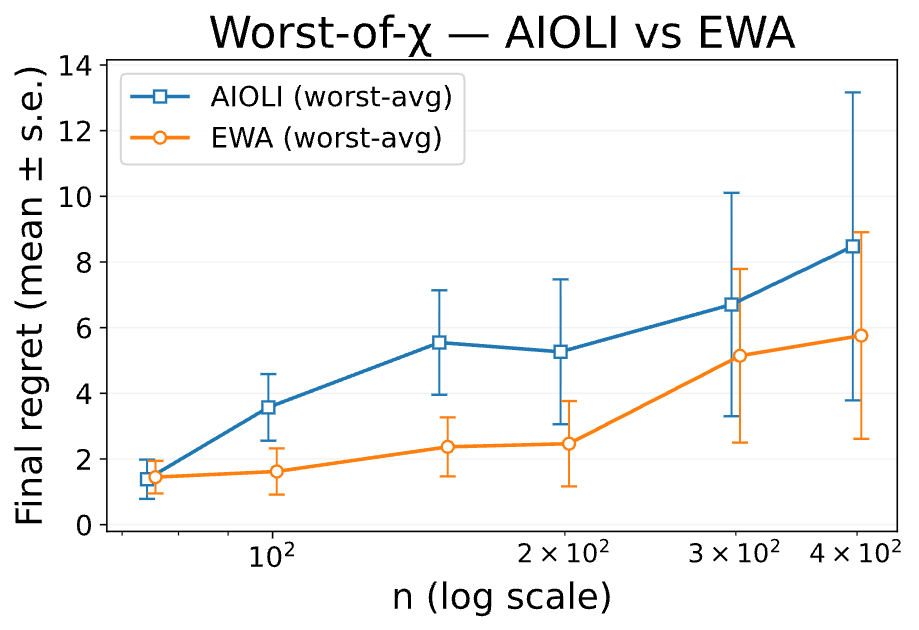}
  \caption{Worst-of-$\chi$ average regret with $\pm$ s.e. bars over $70$ runs on the adversarial data generating process presented in \citet{hazan2014logistic}.}
  \label{fig:aioli-ewa}
\end{figure}

We then plot in Figure \ref{fig:aioli-ewa} the worst of these two average regrets obtained by our EW (orange) and AIOLI (blue) \citep{rudi_binary_efficient_logreg} according to the value of $n$ (log-scale) and together with the corresponding standard errors. We observe that EW is slightly better than AIOLI; furthermore, the implementation of EW we adopted is computationally less expensive than the worst-case implementation proposed in the paper as we discuss in the next subsection.

\subsection{EW (cheaper) implementation}
To obtain a performance comparable to AIOLI (see Figure \ref{fig:aioli-ewa}), we relax the \emph{worst-case} sampling prescription used in the theory. In particular, instead of running the theoretically prescribed number of independent chains, we approximate the EW predictive distribution with a \emph{single} adaptive MALA chain targeting
\[
\rho_t(\theta)\propto \exp\Big(-\frac{\theta^2}{2B^2}\Big)\prod_{i<t}\sigma(y_i\,\theta x_i).
\]
Accordingly, the samples used in this subsection are Markov-dependent, so the i.i.d. concentration argument underlying the theorem does not apply. At round \(t\), we warm-start the chain at the last state from round \(t-1\) (or from \(\mathcal N(0,B^2)\) at \(t=1\)), and initialize the step size as
\[
h_t=\Big(10^{-3}+\frac14 R^2(t-1)+B^{-2}\Big)^{-1}.
\]
We then run a short pilot phase (typically 4--6 proposals) and adapt \(h_t\) multiplicatively until the acceptance rate falls in \([0.55,0.80]\). After that, we use a small constant burn-in (e.g., 8--12 MALA steps) and retain \(S=24\) samples by continuing the same chain, optionally with thinning \(\tau\in\{1,2\}\). The predictive probability is estimated by
\[
\hat p_t=\frac1S\sum_{k=1}^S \sigma\big(y_t\,\theta^{(k)}x_t\big),
\]
and the incurred loss is \(-\log \hat p_t\).

This implementation should be viewed purely as a practical heuristic. It is distinct from the worst-case construction analyzed in \cref{th: regret guarantee of COMPUTABLE EWA}: there, the quantity \(s_t\) denotes the number of \emph{independent} output samples required at round \(t\), and the corresponding complexity bound is stated in terms of the number of first-order oracle queries. Here, by contrast, we run one adaptive chain and retain several \emph{dependent} samples; accordingly, we discuss its arithmetic runtime separately.

Let \(M_t\) denote the total number of MALA transitions performed at round \(t\), including pilot adaptation, burn-in, thinning, and the transitions needed to produce the retained samples used in the Monte Carlo average. In our one-dimensional strongly log-concave setting, the posterior changes only mildly from one round to the next, so a warm-started chain typically mixes rapidly in practice. This motivates using fixed values of \(S\), \(\tau\), and burn-in, together with a constant transition budget \(M_t\). \\
At round \(t\), evaluating \(\nabla \log \rho_t(\theta)\) and \(\log \rho_t(\theta)\) requires one pass over the prefix \(\{(x_i,y_i)\}_{i=1}^{t-1}\); in \(d\) dimensions this costs \(\Theta((t-1)d)\) arithmetic operations, and in the one-dimensional setting considered here it is simply \(\Theta(t-1)\). Since one MALA transition requires a constant number of such evaluations, each transition costs \(\Theta((t-1)d)\). Therefore, if we perform \(M_t\) transitions at round \(t\), the arithmetic cost of round \(t\) is
\[
\Theta\big(M_t (t-1)d\big),
\]
and the total arithmetic cost up to time \(n\) is
\[
T_{\mathrm{EW}}(n)
=
\Theta\Bigg(d\sum_{t=1}^n M_t (t-1)\Bigg).
\]

With a fixed practical budget \(M_t\le M\) independent of \(t\), this yields
\[
T_{\mathrm{EW}}(n)=\Theta(d\,M\,n^2)=\Theta(d\,n^2),
\]
where the last equality hides only the constant \(M\). In particular, in one dimension we obtain \(\Theta(n^2)\). More generally, if \(M_t=\Theta(t^\alpha)\), then
\[
T_{\mathrm{EW}}(n)=\Theta\big(d\,n^{2+\alpha}\big).
\]

We stress that this arithmetic-runtime estimate is \emph{not} directly comparable to the worst-case \(\widetilde O(B^3n^5)\) guarantee in \cref{th: regret guarantee of COMPUTABLE EWA}, since the latter is a bound on total first-order oracle complexity for the independent-sample implementation analyzed in the theory. Rather, the point of this subsection is that a warm-started single-chain heuristic can be implemented much more cheaply in practice while preserving similar predictive behavior in our experiments.

\end{document}